\documentclass[12pt]{article}
\usepackage{arxiv}

\usepackage{amsmath, amsthm, amssymb, graphicx, url}
\usepackage[algo2e,ruled]{algorithm2e}
\usepackage[compress,sort,numbers]{natbib}

\RequirePackage{placeins}
\PassOptionsToPackage{x11names}{xcolor}
\RequirePackage{xcolor}
\RequirePackage{algorithm2e}
\usepackage{wasysym}

\usepackage[utf8]{inputenc} 
\usepackage[T1]{fontenc}    
\usepackage{hyperref}       
\usepackage{url}            
\usepackage{booktabs}       
\usepackage{amsfonts}       
\usepackage{nicefrac}       
\usepackage{microtype}      
\usepackage{xcolor}         

\usepackage{amsmath, amssymb, natbib, graphicx, url, algorithm2e}
\usepackage{amssymb}
\usepackage{natbib}

\usepackage{bbm}
\usepackage{mathtools,  verbatim}
\usepackage{amsthm,amssymb}

\usepackage{wasysym}

\usepackage{hyperref}
\usepackage[symbol]{footmisc}

\usepackage{thm-restate}
\usepackage{thmtools}

\newtheorem{theorem}{Theorem}

\newtheorem{lemma}{Lemma}

\newtheorem{definition}{Definition}

\newtheorem{propositionc}[theorem]{Proposition}

\newtheorem{construction}{Construction}

\usepackage{comment}
\usepackage{algorithm}
\usepackage{algpseudocode}
\usepackage{bbm, dsfont}

\usepackage{graphicx}
\usepackage{subcaption}
\usepackage{booktabs}

\newcommand{\modeprop}{\gamma^{\mathrm{mode}}}

\newcommand{\maplink}{\psi^{\varphi}}
\newcommand{\maplinklevel}{\psi^{\varphi}_{y}}
\newcommand{\maplinknoise}{\psi^{P , \alpha}}

\DeclareMathOperator*{\argmax}{arg\,max}
\DeclareMathOperator*{\argmin}{arg\,min}

\newcommand{\reals}{\mathbb{R}}

\newcommand{\prop}{\mathrm{prop}}

\newcommand{\concvx}{\mathrm{cons}_\mathrm{cvx}}

\newcommand{\eliccvx}{\mathrm{elic}_\mathrm{cvx}}

\newcommand{\affhull}{\mathrm{affhull}}
\newcommand{\neigh }{\mathrm{ne}}

\renewcommand{\dim}{\mathrm{dim}}

\newcommand{\simplex}{\Delta_\Y}

\newcommand{\B}{\mathcal{B}}
\newcommand{\C}{\mathcal{C}}
\newcommand{\D}{\mathcal{D}}
\newcommand{\E}{\mathbb{E}}

\renewcommand{\H}{\mathcal{H}}

\renewcommand{\L}{\mathcal{L}}

\renewcommand{\P}{\mathcal{P}}
\newcommand{\R}{\mathcal{Y}}
\newcommand{\Sc}{\mathcal{S}}

\newcommand{\X}{\mathcal{X}}
\newcommand{\Y}{\mathcal{Y}}

\newcommand{\toto}{\rightrightarrows}

\newcommand{\conv}{\mathrm{conv}\,}

\newcommand{\ones}{\mathbbm{1}}

\renewcommand{\emptyset}{\varnothing}

\newcommand{\vertex}{\mathrm{vert}}

\newcommand{\haty}{\hat{y}}

\newcommand{\ignore}[1]{}

\usepackage{times}
\usepackage{color}
\usepackage{soul}

\definecolor{darkblue}{rgb}{0.0,0.0,0.2}
\definecolor{darkgreen}{rgb}{0.0,0.3,0.0}
\hypersetup{colorlinks,breaklinks,
	linkcolor=darkblue,urlcolor=darkblue,
	anchorcolor=darkblue,citecolor=darkblue}
\usepackage[colorinlistoftodos,textsize=tiny]{todonotes} 

\newcommand{\Comments}{0}
\newcommand{\mynote}[2]{\ifnum\Comments=1\textcolor{#1}{#2}\fi}
\newcommand{\mytodo}[2]{\ifnum\Comments=1	\todo[linecolor=#1!80!black,backgroundcolor=#1,bordercolor=#1!80!black]{#2}\fi}

\newcommand{\dk}[1]{\mynote{purple}{[DK: #1]}}

\newcommand{\journal}[1]{}

\ifnum\Comments=0            
\fi

\title{Trading off Consistency and Dimensionality of Convex Surrogates for the Mode}

\author{
Enrique Nueve\\
Department of Computer Science\\
University of Colorado Boulder\\
\texttt{enrique.nueveiv@colorado.edu} \\
\And
Bo Waggoner \\
Department of Computer Science\\
University of Colorado Boulder \\
\texttt{bwag@colorado.edu} \\
\AND
Dhamma Kimpara \\
Department of Computer Science\\
University of Colorado Boulder\\
\texttt{dhamma.kimpara@colorado.edu} \\
\And
Jessie Finocchiaro\thanks{Most of this work was completed while author was at Harvard University CRCS} \\
Department of Computer Science\\
Boston College \\
\texttt{finocch@bc.edu} 
}

\begin{document}

\maketitle

\begin{abstract}
In multiclass classification 
over $n$ outcomes, we typically optimize some \emph{surrogate loss} $L: \reals^d \times\Y \to \reals$ assigning real-valued error to predictions in $\reals^d$.
In this paradigm, outcomes must be embedded into the reals with dimension $d \approx n$ in order to design a \emph{consistent} surrogate loss.
Consistent losses are well-motivated theoretically, yet for large $n$, such as in information retrieval and structured prediction tasks, their optimization may be computationally infeasible.
In practice, outcomes are typically embedded into some $\reals^d$ for $d \ll n$, with little known about their suitability for multiclass classification.
We investigate two approaches for trading off consistency and dimensionality in multiclass classification while using a convex surrogate loss.
We first formalize \emph{partial consistency} when the optimized surrogate has dimension $d \ll n$. 
We then check if partial consistency holds under a given embedding and low-noise assumption, providing insight into when to use a particular embedding into $\reals^d$. 
Finally, we present a new method to construct (fully) consistent losses with $d \ll n$ out of multiple problem instances.
Our practical approach leverages parallelism to sidestep lower bounds on $d$.
\end{abstract}

\section{Introduction}

Multiclass classification, due to its combinatorial and discontinuous nature, is intractable to optimize directly, which drives machine learners to optimize some nicer \emph{surrogate loss}. 
To ensure these surrogates properly ``correspond'' to the discrete classification task, we seek to design \emph{consistent} surrogates.
If one uses a consistent surrogate loss, in the limit of infinite data and model expressivity, one ends up with the same classifications as if one had solved the original intractable problem directly with probability $1$.

Surrogate losses form the backbone of gradient-based optimization for classification tasks. Optimizing a surrogate is easier than direct optimization, but a large dimension $d$ of the surrogate loss $L:\reals^{d}\times \Y\to\reals$ can make gradient-based optimization intractable.
Therefore, previous literature has operated under the premise that the prediction dimension $d$ should be as low as possible, subject to consistency for the classification task~\citep{ramaswamy2016convex,finocchiaro2024embedding,finocchiaro2020embedding}. 
For multi-class classification over $n$ outcomes, the lower bound on $d$ is $n-1$ \citep{ramaswamy2016convex}.

These previous works implicitly focus on a binary approach to consistency: a surrogate is either consistent for every possible label distribution, or it is not consistent. 
But there is a way out: lower bounds on the surrogate dimension $d$ rely on edge-cases that rarely show up in reality \citep{ramaswamy2016convex}.
As a result, practitioners are often willing to trade-off the guarantee of consistency in order to improve the computational tractability of optimization.
However, we currently lack rigorous analysis tools to analyze many of the partially-consistent surrogates commonly used in practice. 
Thus, \emph{unlike previous works, our work focuses on this more realistic paradigm of partial consistency.}
We apply our unique approach to rigorously analyze a popular surrogate construction that encompasses methods such as one-hot and binary encoding.
Our approach allows for fine-grained control of the trade-off between consistency and dimension.

Prior works have informally brushed upon the proposed partial-consistency paradigm, without rigorous study.
For example, \citet{agarwal2015consistent} impose a low-noise assumption to construct a surrogate for classification with $d = \log(n)$.
However, their work does not provide any way to control the consistency-dimension trade-off.
Similarly, \citet{struminsky2018quantifying} characterize the excess risk bounds of inconsistent surrogates, which teaches us about the learning rates for inconsistent surrogates, but not \emph{under which distributional assumptions} we can recover consistency guarantees.

Using different techniques than both of these approaches, we seek to understand the tradeoffs of consistency, surrogate prediction dimension, and number of problem instances through the use of polytope embeddings which are common in the literature \citep{wainwright2008graphical,blondel2020learning}. 
When embedding outcomes into $d \ll n$ dimensions, we first show there always exists a set of distributions where \emph{hallucinations} occur: where the report minimizing the surrogate leads to a prediction $\hat y$ such that the underlying true distribution has no weight on the prediction; that is, $Pr[Y = \hat y] = 0$~ (Theorem~\ref{thm:halreal}).
Following this, we show that every polytope embedding is partially consistent under strong enough low-noise assumptions (Theorem~\ref{thm:lownoisexists}).
Finally, we demonstrate through leveraging the embedding structure and multiple problem instances that the mode (in particular, a full rank ordering) over $n$ outcomes embedded into a $\frac{n}{2}$ dimensional surrogate space is elicitable over all distributions via $O(n^2)$ problem instances (Theorem~\ref{thm:multi-instance}).
This alternative approach to recovering consistency is parallelizable, detangling the complexity of gradient computation of one high-dimensional surrogate.

\section{Background and Notation}
Let $\Y$ be a finite label space, and throughout let $n = |\Y |$.
Define $\reals_{+}^{\Y}$ to be the nonnegative orthant.
Let $\Delta_{\Y}=\{p\in \reals^{\Y}_{+}\mid \|p\|_{1}=1  \}$ 
be the set of probability distributions on $\Y$, represented as vectors.
We denote the point mass distribution of an outcome $y\in\Y$ by $\delta_{y}\in\Delta_{\Y}$.
Let $[d]:= \{1,\dots ,d\}$.
In general, we denote a discrete loss by $\ell :\Y \times \Y \to \reals_{+}$ with outcomes denoted by $y\in \Y$ and a surrogate loss by
$L:\reals^{d}\times \Y \to \reals$ with surrogate reports $u\in\reals^{d}$ and outcomes $y\in\Y$. 
The surrogate must be accompanied by a link $\psi : \reals^d \to \R$ mapping the convex surrogate model's predictions back into the discrete target space, and we discuss consistency of a \emph{pair} $(L, \psi)$ with respect to the target $\ell$.

For $\epsilon > 0$, we define an epsilon ball via $B_{\epsilon}(u)=\{x\in\reals^{d}\mid \|u-x\|_{2}<\epsilon  \}$ and $B_\epsilon := B_\epsilon(\vec 0)$.
Given a closed convex set $\C \subset \reals^{d}$, we define a projection operation onto $\C$ via $\text{Proj}_{\C}(u) := \argmin_{x\in \C} \|u-x\|_{2}$.
Given a closed convex set $\C \subset \reals^{d}$ and $u\in \reals^{d}$, we let the set-pointwise distance to be defined as $\|u-\C \|_{2} := \|u- \text{Proj}_{\C}(u) \|_{2}$.
Full tables of notation are found in Appendix~\ref{app:notation}.



\subsection{Property Elicitation, Consistency, and Prediction Dimension}
Discrete label prediction requires optimization of a target loss function, $\ell$, e.g. multi-class classification and 0-1 loss.
When designing surrogate losses, consistency is the key notion of correspondence between surrogate and target loss.
Intuitively, consistency implies that minimizing surrogate risk corresponds to solving the target problem.
\citet{finocchiaro2021unifying} show that surrogate loss consistency is a necessary precursor to excess risk bounds and convergence rates.

Consistency is generally a difficult condition to work with directly.
Hence, we will use the notion of \emph{calibration}, which is equivalent to consistency in our setting with finite outcomes.
Our approach follows from the property elicitation literature, which allows us to abstract away from the feature space $\mathcal{X}$ and focus on the conditional distributions over the labels,  $p = \Pr[Y \mid X = x] \in \simplex $ \citep{bartlett2006convexity,tewari2007consistency,zhang2004statistical,ramaswamy2016convex,steinwart2007compare}. 
In this approach, the central object of study is a \emph{property} which maps label distributions to reports that minimize the loss.

\begin{definition}[Property, Elicits, Level Set]\label{def:elicit}
Let $\mathcal{R}$ be an arbitrary report set.
For $\mathcal{P}\subseteq \Delta_{\Y}$,
    a property is a set-valued function $\Gamma :\mathcal{P}  \to 2^\mathcal{R}\setminus \{\emptyset \}$, which we denote $\Gamma : \mathcal{P}\toto \R$. 
   A loss $L:\mathcal{R} \times \Y\to  \reals$ elicits the property $\Gamma$ on $\P$ if 
   $$\forall \; p \in \mathcal{P}, \; \Gamma (p)=\argmin_{u\in\mathcal{R}}\E_{Y\sim p} [L(u,Y)]~.$$
If $ L$ elicits a property, it is unique and we denote it $\prop [L]$.
The level set of $\Gamma$ for report $r$ is the set $\Gamma_r := \{ p\in\mathcal{P} \mid r = \Gamma (p) \}$. 
If $\prop [L]=\Gamma$ and $|\Gamma (p)|=1$ for all $p\in  \P$, we say that $L$ is strictly proper for $\Gamma$.
\end{definition}
In this work, $\mathcal{R} =\Y$ for target losses and $\mathcal{R} = \reals^{d}$ for surrogate losses.  


Once a model is optimized wrt. a surrogate $L$, it predicts reports in the surrogate space, $\reals^d$.
Then, to map surrogate reports to discrete labels, the surrogate loss must be paired with a link, $\psi : \reals^d \to \R$.
Intuitively, a surrogate and link pair $(L,\psi)$ are calibrated with respect to a target loss $\ell$, if the optimal expected surrogate loss when making the \emph{incorrect classification} (by $\psi$) is strictly greater than the optimal surrogate loss.


\begin{definition}[$\ell$-Calibrated Loss]
Given discrete loss $\ell : \R\times \Y \to \reals_{+}$, surrogate loss $L :\reals^{d}\times \Y  \to \reals$, and link function $\psi :\reals^{d}\to \R$.
We say that  $(L,\psi )$ is $\ell $-calibrated over $\P \subseteq \simplex$ if, for all $p\in\mathcal{P}$,
$$\inf_{u\in\reals^{d}:\psi (u)\notin \prop [\ell ](p) } \E_{Y\sim p} [L(u,Y)]
 > \inf_{u\in\reals^{d}} \E_{Y\sim p} [L(u,Y)]~.$$
 If $\P$ is not specified, then we are discussing calibration over $\simplex$. 
 In general, when $(L,\psi )$ is $\ell $-calibrated over $\P$ such that $\P \subset \simplex$, we say partial calibration holds with respect to $\P$.
\end{definition}


Our analysis crucially relies on the ability to specify $\P$ when invoking the definition of calibration.
This is because the surrogates we analyze break the $d=n-1$ lower bound on the dimension of any consistent surrogate loss. 
So the surrogates will not be calibrated over the whole simplex $\simplex$.
To aid in our analysis, we use a condition that shows that converging to a property value implies calibration for the target loss itself \citep{agarwal2015consistent}.


\begin{definition}[$\ell $-Calibrated Property]
Let $\mathcal{P}\subseteq \Delta_{\Y}$, $\Gamma :\mathcal{P}\toto \reals^{d}$, discrete loss $\ell : \R\times \Y \to \reals_{+}$, and $\psi :\reals^{d}\to \R$. 
We will say $(\Gamma ,\psi)$ is $\ell $-calibrated for all $p\in\mathcal{P}$ and all sequences in $\{u_m\}$ in $\reals^{d}$ if,
$$u_m \to \Gamma (p) \Rightarrow \E_{Y\sim p}[\ell (\psi (u_m), Y)]\to \min_{r\in \R} \E_{Y\sim p} [\ell (r,Y)]~.$$
\end{definition}




\begin{theorem}[{\citep[Theorem 3]{agarwal2015consistent}}]\label{thm:conspropcal}
 Let $\ell :\R\times \Y \to \reals_{+}$ and $\P\subseteq \Delta_{\Y}$.
Let $\Gamma :\mathcal{P}\toto \reals^{d}$ and $\psi :\reals^{d}\to \R$ be such that $\Gamma $ is elicitable and $(\Gamma ,\psi )$ is an $\ell $-calibrated property over $\P$.
Let $L:\reals^d \times \Y \to \reals$ be a convex function for all $y\in \Y$ and  strictly proper for $\Gamma$ i.e. $\prop[L] = \Gamma$ and $|\Gamma (p)|=1$ for all $p\in  \P$.
Then, $(L, \psi)$ is $\ell $-calibrated over $\P$.

\end{theorem}






\noindent Finally, we present the 0-1 loss that we analyze, which is the target loss for multiclass classification.

\begin{definition}[0-1 Loss]\label{def:01_loss}
We denote the 0-1 loss by $\ell_{0-1 }:\Y \times \Y \to \{ 0,1  \}$ such that
 $\ell_{0-1}(y, \hat{y}) := \ones_{y \neq \hat{y}}~$.
Observe $\modeprop (p):=\prop [\ell _{0-1}] (p)=\{ y\in\Y | y \in  \argmax_y p_y   \}$.
\end{definition}

 







\section{Polytope Embedding and Existence of Calibrated Regions}\label{sec:calregions}

Often, discrete outcomes are embedded in continuous space onto the vertices of the simplex via one-hot encoding, or the vertices of the unit cube via binary encoding \citep{seger2018investigation}. 
Generalizing, we introduce an approach to surrogate construction inspired by \citet{wainwright2008graphical} and \citet{blondel2020learning} that 
encompasses the aforementioned embedding methods.
This construction utilizes embeddings onto the vertices of arbitrary low-dimensional polytopes $\varphi : \R \to \reals^d$.
Then, an embedding scheme naturally induces a class of loss functions $L^{2}_\varphi$ defined by the embedding, and a link function $\maplink$.

Our analysis begins by defining a condition stronger than inconsistency that arises when embedding into $d < n-1$ dimensions for multiclass classification.
To this end, we introduce the notion of \emph{hallucination} as a means to characterize the ``worst case'' behavior of a surrogate pair (\S~\ref{subsec:hallucination-region}).
In a positive manner, we characterize the \emph{calibration regions} of various embeddings (\S~\ref{subsec:calibration-region}), which are sets $\P \subseteq \simplex$ such that our surrogate and link pair $(L^{2}_\varphi, \maplink)$ are $\ell$-calibrated over $\P$.
We refer the reader to the Appendix \ref{app:strict} for omitted full proofs. 

\subsection{Polytope Embedding Construction}\label{subsec:polytope-embedding}

A Convex Polytope $P\subset \reals^{d}$, or simply a polytope, is the convex hull of a finite number of points $u_1,\dots ,u_n\in\reals^{d}$.
An extreme point of a convex set $A$, is a point $u\in A$ such that if $ u = \lambda y + (1 - \lambda  )z$ with $y,z \in A$ and $\lambda \in [0, 1]$, then $y = u$ and/or $z = u$. 
We shall denote by $\vertex (P)$ a polytope's set of extreme points.
A polytope can be expressed by the convex hull of its extreme points, i.e. $P=\conv (\vertex (P) )$ \citep[Theorem 7.2]{brondsted2012introduction}.
Additional definitions pertaining to polytopes are used for proofs that are omitted to the appendix, we refer the reader to (\S~\ref{app:polyapp}) for said definitions. 

We propose the following embedding procedure that allows one to construct surrogate losses with almost \emph{any} polytope, and \emph{any} Bregman divergence.

\begin{construction}[Polytope Embedding]\label{cvxregpolyembedgen}
Given $\Y$ outcomes,  $|\Y | =n$, choose a polytope 
$P \subset \reals^{d}$ such that $|\vertex (P)|=n$.
Choose a bijection between $\Y$ and $\vertex(P)$.
According to this bijection, assign each vertex a unique outcome so that $\{v_y \in \reals^{d}| y \in \Y\} = \vertex(P)$.
Then the polytope embedding $\varphi: \simplex \to P$ is
$\varphi (p) := \sum_{y \in \Y} p_{y} v_y$, which is the sum of $p$-scaled vectors 
\end{construction}


Following the work of \cite{blondel2019structured} and their proposed Projection-based losses, we use Square Losses and a polytope embedding $\varphi$ to define an induced loss $L^{2}_{\varphi}$ (Definition \ref{def:indloss}).



\begin{definition}[Square Loss]\label{def:sqloss}
    Given a constant $c\in\reals_{>0}$ and a function $f:\mathcal{Y}\to\reals_+$, we define the square loss $L^2:\reals^d\times \mathcal{Y}\to \reals_+$ as $$L^2(u,y) = \frac{1}{c}\|u-y\|^{2}_2+f(y)~.$$
\end{definition}


\begin{definition}[$(L^2 ,\varphi )$ Induced Loss]\label{def:indloss}
   Given a square loss $L^2$ and a polytope embedding $\varphi$, we say $(L^2,\varphi  )$ 
   induces a loss $L_{\varphi}^{2}:  \reals^{d}\times \Y \to \reals_{+}$ defined as
   $$ L_{\varphi}^{2}(u, y):= \frac{1}{c}\|u-v_y\|^{2}_2+f(y) ~.$$ 
   Note, that for any fixed $y\in\Y$, $L_{\varphi}^{2}(u, y)$ is convex with respect to $u\in\reals^{d}$. 
\end{definition}



We show that for any $p\in \Delta_{\Y}$, the report that uniquely minimizes the expectation of the loss $L^{2}_{\varphi}$ is $\varphi(p)$, the embedding point of $p$.
Furthermore, the polytope $P$ contains  all of, and only the minimizing reports in expectation under $L^{2}_{\varphi}$. 


\ignore{\begin{propositionc}\label{thm:embedmin}\end{propositionc}}\begin{restatable}{propositionc}{embedmin}\label{thm:embedmin}
For a given induced loss $L_{\varphi}^{2}$, the unique report which minimizes the expected loss is $ u^{*} := \argmin_{u\in\reals^{d}} \E_{Y\sim p}[L_{\varphi}^{2}(u,Y) ] = \varphi (p) $ such that $u^{*}\in P$. 
Furthermore, every $\hat{u}\in P$ is a minimizer of $\E_{Y\sim\hat{p}} [L_{\varphi}^{2}(u,Y)]$ for some $\hat{p}\in\Delta_{\Y}$.
\end{restatable}


We now define the maximum a posteriori (MAP) link, which will be used in conjunction with an induced loss $L^{2}_{\varphi}$ to form a surrogate pair for the 0-1 loss.
The MAP link projects surrogate predictions onto the polytope $P$, then links to the nearest vertex of $P$, and is commonly used in the literature~\citep{tsochantaridis2005large,blondel2019structured,xue2016solving}.
Since the MAP link performs a projection, one may ask if this is computationally challenging; fortunately this operation is computationally feasible due to the convexity of the polytope~\citep{blondel2019structured}.


\begin{definition}[MAP Link]
Let $\varphi$ be a polytope embedding from $\Delta_{\Y}$ to $P$. The MAP link $\maplink  : \reals^{d} \to \Y$ is defined as
$\maplink (u) = \argmin_{y\in \Y   }  ||\text{Proj}_{P}(u) -v_y||_{2}$.
The level set of the link for $y$ is $\maplinklevel =\{ u\in\reals^{d}|y =  \maplink  (u) \}$.
We break ties arbitrarily but deterministically.
\end{definition}


\subsection{Hallucination Regions}\label{subsec:hallucination-region}
Since our polytope embedding violates surrogate dimension bounds, calibration for 0-1 loss will not hold for all distributions.
In particular, we show there always exists some distribution $p$ such that $p_y = 0$ yet $\E_{Y \sim p}[L_{\varphi}^2 (u,Y)]$ is minimized at some $u$ such that $\psi^\varphi(u) = y$.
This implies a ``worst case'' inconsistency where the reported outcome could never actually occur with respect to our embedding of $n$ events via $\varphi$ into $\vertex (P)$.

\begin{definition}[Hallucination]\label{def:hallucination}
Given $(L, \psi )$ such that $L :\reals^{d} \times \Y \to \reals_{+}$, $|\Y |=n$, $d<n$, and $\psi:\reals^{d}\to \Y$, we say that a hallucination occurs at a surrogate report $u\in\reals^{d}$ if, for some $p\in\Delta_{\Y}$, $u \in \argmin_{\hat{u}\in\reals^d} \E_{\Y\sim p} [L(\hat{u}, Y )]$ and $\psi(u) := y$ but $p_y = 0$.
We denote by $\H \subseteq P \subset \reals^{d}$ as the \textit{hallucination region} as the elements of $P$ at which hallucinations can occur.
\end{definition}


\begin{figure}[t]
\centering

\begin{minipage}[t]{0.05\textwidth}
	\centering
\end{minipage}
\hfill
\begin{minipage}[t]{0.35\textwidth}
	\centering
	\includegraphics[width=.9\textwidth]{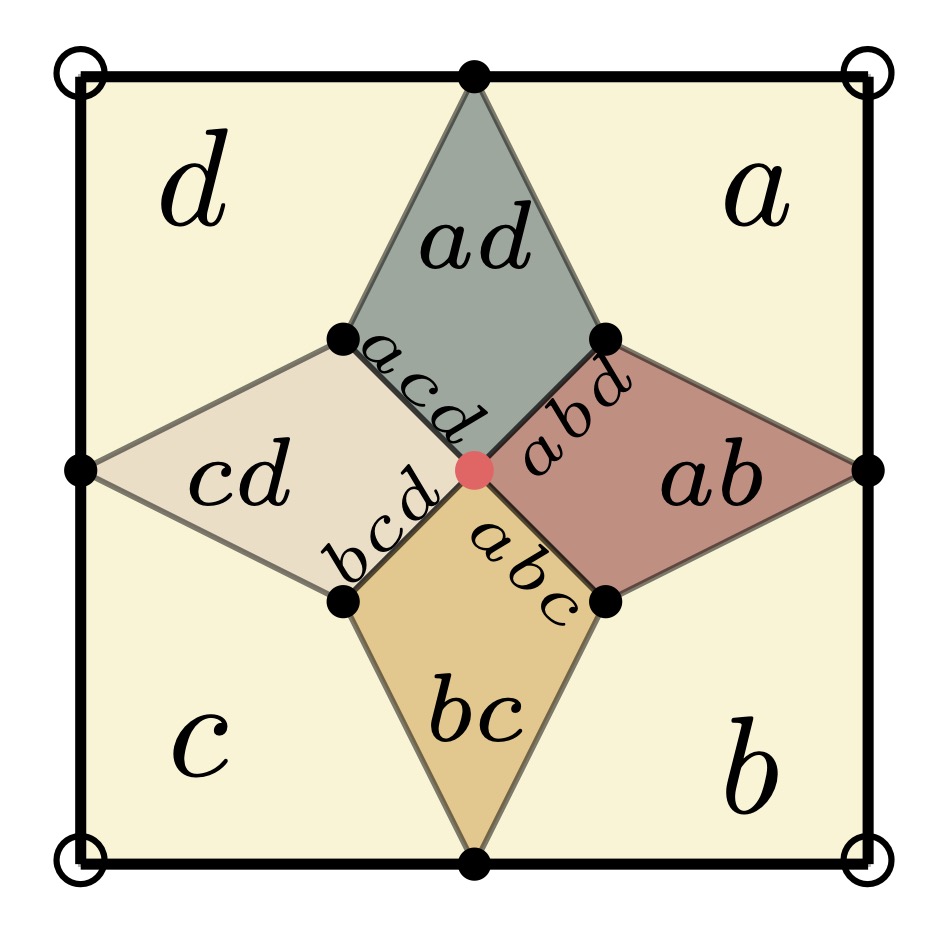}
\end{minipage}
\hfill
  \begin{minipage}[t]{0.45\textwidth}
    \includegraphics[width=.9\textwidth]{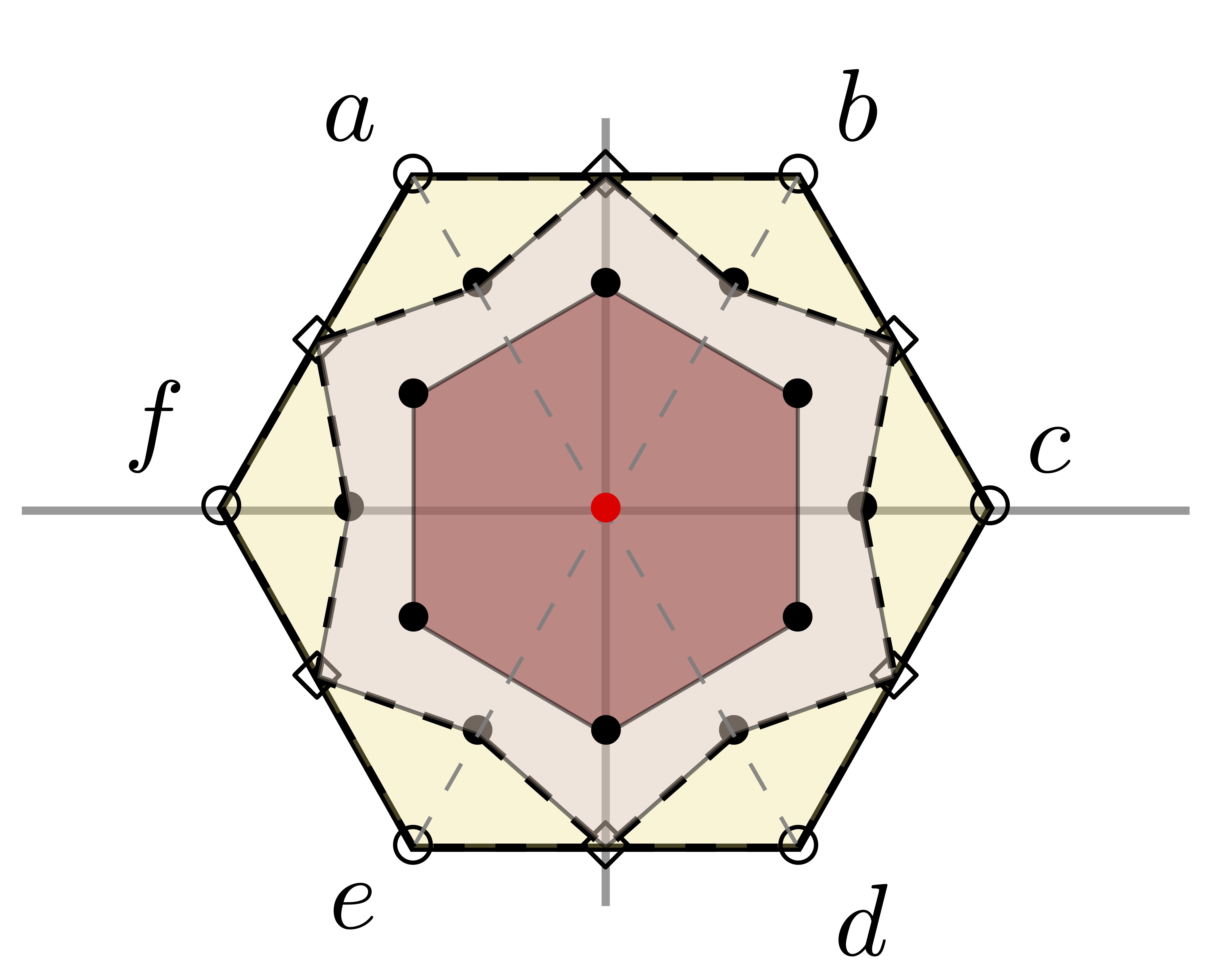}
  \end{minipage}
\begin{minipage}[t]{0.1\textwidth}
	\centering
\end{minipage}

  \caption{(Left) Mode level sets of $\Delta_{\Y}$ where $\Y =\{a,b,c,d\}$ embedded into a two dimensional unit cube. The center red point denotes the origin $(0,0)$ which is the hallucination region. (Right) An embedding of $\Delta_{\Y}$ where $\Y =\{a,b,c,d,e,f\}$ into a three-dimensional permutahedron: the beige region expresses strict calibration regions, the light pink regions expresses regions with inconsistency, and the auburn region expresses regions with hallucinations. For example, consider the report $u = \vec 0$. Since losses are convex, if $p = (0, \frac 1 2, 0, 0, \frac 1 2, 0)$, then $\conv(\{b,e\})$ (dashed grey) is optimal, which includes $u$. However, $\vec 0$ is also contained in $\conv(\{a,d\})$ which is optimal for the distribution $p' = (\frac 1 2, 0, 0, \frac 1 2, 0, 0)$. Therefore, we cannot distinguish the optimal reports for a hallucination at $\vec 0$.} 
\label{fig:example-plots}
\end{figure}


 We express the subspace of the surrogate space where hallucinations can occur as the hallucination region denoted by $\H$.
In Theorem \ref{thm:halreal}, we characterize the hallucination region for any polytope embedding while using the surrogate pair $(L^{2}_{\varphi} ,\maplink )$ and show that $\H$ is never empty.



\begin{restatable}{theoremc}{halreal}
  \label{thm:halreal}
For any given pair $(L^{2}_{\varphi} ,\maplink )$ and $\ell_{0-1}$ with embedding dimension $d < n-1$;    it holds that 
$\H= \cup_{y\in\Y }\conv (\vertex (P)\setminus \{v_y\} )\cap \maplinklevel$
and furthermore $\H \neq \emptyset$.
\end{restatable}\vspace{-.4cm}
\begin{proof}[Sketch]
Fix $y\in\Y$.
We abuse notation and write $\vertex (P_{-y}) := \vertex (P)\setminus \{v_y\}$.
Observe $\conv (\vertex (P_{-y}))\cap \maplinklevel\subseteq \H$ since any point in this set can be expressed as a convex combination without needing vertex $v_{y}$ implying there is a distribution embedded by $\varphi$ to said point which has no weight on $y$.
We seek to show that $\H \subseteq \cup_{y\in\Y } \conv (\vertex (P_{-y}))\cap \maplinklevel$.
Assume there exists a point $u\notin \conv (\vertex (P)\setminus v_y)\cap \maplinklevel$ such that there exists some $p\in\Delta_{\Y}$ where $\varphi (p)=u$, $p_y =0$, and $\maplink (u)=y$.
Since $\maplink (u)=y$ and $u\notin \conv (\vertex (P_{-y}))\cap \maplink_{y}$, it must be the case that $u\notin \conv (\vertex (P_{-y}))$.
However, that implies that $u$ is strictly in the vertex figure and thus must have weight on the coefficient for $y$.
Thus, forming a contradiction that $p_y=0$ which implies that $\H \subseteq \cup_{y\in\Y } \conv (\vertex (P_{-y}))\cap \maplinklevel$.
Finally, using Helly's Theorem \citep[Corollary 21.3.2]{rockafellar1997convex} we show that $\cap_{y\in\Y} \conv (\vertex (P) \setminus v_y) \neq \emptyset$, which implies the non-emptiness of $\H$ as well.
\end{proof}




Theorem~\ref{thm:halreal} suggests that using machine learning in high-risk settings such as medical and legal applications while violating the known $n-1$ dimensional bound for surrogate losses in multiclass classification is inherently ill-advised without human intervention given the possibility for hallucinations. 
Furthermore, hallucinations may be forced by the target loss, as in the case of Hamming loss (see Appendix \ref{app:hal}).
In these cases practitioners should carefully consider the choice of target loss.
We conjecture that hallucinations are common for many structured prediction losses.
However this is not a concern in our primary loss of study of multi-class classification.



\subsection{Calibration Regions}\label{subsec:calibration-region}
Ideally, we would like calibration to hold over the entire simplex since that would imply minimizing surrogate risk would always correspond to solving the target problem regardless of the true underlying distribution.
We observe that the mode's embedded level sets in the polytope overlap (see Figure~\ref{fig:example-plots}L), which is unsurprising given that we are violating the lower bounds on surrogate prediction for the mode and hence calibration does not hold over the entire simplex.
Since $|2^{\Y}\setminus \{ \emptyset \} |$ is a finite set, we know that the number of unique mode level sets is finite. 
Although every point in the polytope is a minimizing report for some distribution, if multiple distributions with non-intersecting mode sets are embedded to the same point, there is no way to define a link function that is correct in all cases. 
However, if the union 
of mode sets for the $p$'s mapped to any $u\in P$ is a singleton, regardless of the underlying distribution\footnote{We leave the more general case of linking $u$ when $\bigcap_{p \in \varphi^{-1}(u)} \gamma(p) \neq \varnothing$ to future work.}, a link $\psi$ would be calibrated over the union if it mapped $u$ to the mentioned singleton. 
Given $(L , \psi )$, $\varphi$, and a target loss $\ell$, we define strict calibrated regions as the points for which calibration holds regardless of the actual distribution realized, which are possible at said points.

\begin{definition}[Strict Calibrated Region]\label{def:scr}
Suppose we are given $(L , \psi )$, $\varphi$, and a target loss $\ell$. We say $R \subseteq P$ is a \emph{strict calibrated region} via $(L , \psi )$ with respect to $\ell$ if $(L , \psi )$ is $\ell $-calibrated for all $p\in \varphi^{-1}(R) := \{p\in\Delta_{\Y} : \varphi(p) \in R\}$. 

For any $y\in\Y$, we define $R_y\coloneqq R \cap \psi_y $.
We let $R_{\Y} :=\cup_{y\in\Y}R_y$.
\end{definition}

By violating lower bounds, we are in a partially consistent paradigm where surrogate reports do not necessarily correspond to a unique distribution $p$.
However, strict calibration regions allow us to check whether or not the loss is calibrated for the distribution $p$ generating the data --- even without explicit access to $p$. 
One simply has to check whether the report $u$ is in $R_{\Y}$.


In Theorem \ref{thm:consvert}, regardless of one's chosen $P$, we show that there always exists a non-zero Lebesgue measurable strict calibration region and that $(L^{2}_{\varphi},\maplink)$ is calibrated for the 0-1 loss overall distributions embedded into the strict calibration region. 
This result shows that our surrogate and link construction for \emph{any $d$}, always yields discernible calibration regions --- lending support to the practical use and study of these surrogates.


\begin{restatable}{theoremc}{consvert}
  \label{thm:consvert}
    Let $L^2$ be a Square Loss, $\varphi$ be any polytope embedding, $\maplink$ be the MAP link, and $L_{\varphi}^{2}$ be the loss induced by $(L^2, \varphi)$.
    There exists a $\P \subseteq \Delta_{\Y}$ with non-zero Lebesgue measure such that $\varphi (\P)$ is a strict calibration region via $(L_{\varphi}^{2}, \maplink )$ with respect to $\ell_{0-1}$.
\end{restatable}

\begin{proof}
Recall that $\modeprop (p) := \prop [\ell_{0-1}](p)=\text{mode}(p)$.
By Lemma \ref{lemma:bdpw}, it can be inferred that for any $y\in\Y$ it holds that $\conv (\{v_y \}\cup m_{v_y,\alpha})\subseteq \varphi (\modeprop_y)$ where $m_{v_y,\alpha }:= \{(1-\alpha )v_y+\alpha \overline{v}\mid \overline{v}\in \neigh (v_y) \}$ defined by any $\alpha \in (0,.5)$.
We seek to show that there exists an open ball at $v_y$ such that $(B_{\epsilon}(v_y)\cap P)\cap \varphi (\modeprop_{\hat{y}})=\emptyset$ for all $\hat{y}\neq y$ and $B_{\epsilon}(v_y)\cap P\subset \conv (\{v_y \}\cup m_{v_y,\alpha})\subseteq \varphi (\modeprop_y)$.

Fix $y\in\Y$.
For contradiction, assume for any $\hat{y}\in\Y$ where $ \hat{y}\neq y$, it holds that $B_{\epsilon}(v_y) \cap \varphi (\modeprop_{\hat{y}})\neq \emptyset$ for all $\epsilon >0$.
Observe via $\varphi$ only $\delta_y\in \Delta_\Y$ is embedded at vertex $v_y$.
Since $\modeprop_{\hat{y}}\subset \Delta_\Y$ is closed and convex, by Lemma \ref{lem:polytopoly}, $\varphi (\modeprop_{\hat{y}})\subset P$ is closed and convex and has no intersection with $v_y$. 
Thus, with respect to every $\varphi (\modeprop_{\hat{y}})$ there must exist some non-zero min distance between $v_y$, which we shall denote by $d_{v_y}$. 
For any $\epsilon \in (0,d_{v_y})$, we can define $B_{\epsilon}(v_y)$ such that $B_{\epsilon}(v_y)\cap \varphi (\modeprop_{\hat{y}}) = \emptyset$, forming a contradiction.
Furthermore, by definition of $\conv (\{v_y \}\cup m_{v_y,\alpha})$, we can define a $d_{v_y}'\in (0,d_{v_y})$ such that for any $\epsilon \in (0,\min{(d_{v_y},d_{v_y}')})$
both $B_{\epsilon}(v_y)\cap P\subset \conv (\{v_y \}\cup m_{v_y,\alpha})\subseteq \varphi (\modeprop_y)$ and $(B_{\epsilon}(v_y)\cap P)\cap \varphi (\modeprop_{\hat{y}})=\emptyset$ for all $\hat{y}\neq y$ holds true. 

For each $v_y\in\vertex (P)$ define a $d_{v_y}$ and $d_{v_y}'$ and let $\epsilon' \in \cap_{v_y\in\vertex (P)} (0,\min{(d_{v_y},d_{v_y}')})$.
By the construction of $P$ and the definition of $\maplink$, there exists a $\epsilon''>0$ such that for all $u\in B_{\epsilon''}(v_y)$ it holds that $\psi (u)=y$ and $B_{\epsilon''}(v_y)\subset \maplink_{y}$ .
By the construction of our epsilon ball, for any $y\in\Y$, we claim that $B_{\min \{ \epsilon ',\epsilon ''\} }(v_y)\cap P$ defines a strict calibration region such that $\varphi^{-1}(B_{\min \{\epsilon ',\epsilon ''\} }(v_y)\cap P)$ is a set of distributions for which calibration holds.

Fix $y\in\Y$.
For $p\in \Delta_{\Y}$, suppose a sequence $\{u_m\}$ converges to $\prop [L_{\varphi}^{2}] (p) = \varphi (p)\in B_{\min \{ \epsilon ',\epsilon ''\} }(v_y)\cap P$ (equality by Proposition \ref{thm:embedmin}).
By construction of $B_{\min \{ \epsilon ',\epsilon ''\} }(v_y)\cap P$, $\varphi(p)\in \varphi (\modeprop_{y})$ and $\varphi(p)\notin \varphi (\modeprop_{\hat{y}})$ for all $\hat{y}\neq y$ and thus, $y$ is a minimizing report for $\ell_{0-1} (y;p)$.
Furthermore, since $ B_{\min \{ \epsilon ',\epsilon ''\} }(v_y)\subset \maplink_{y}$, all elements within $B_{\min \{ \epsilon ',\epsilon ''\} }(v_y)$ link to $y$.
Since $\{u_m\}$ converges to $\prop [L_{\varphi}^{2}] (p)$, there exists some $N\in \mathbb{N}$ and $n\geq N$, such that $\|u_n -  \varphi (p)   \|_{2}< \min \{ \epsilon ',\epsilon ''\} $, meaning that  $\E_{\Y\sim p}[\ell_{0-1} (\maplink (u_m), Y)]\to \min_{y\in\Y }\E_{\Y\sim p}[\ell_{0-1} (y, Y)]$.
Hence, for any $v_y\in\vertex (P)$, $(\prop [L_{\varphi}^{2}],\maplink )$ is $\ell_{0-1}$-calibrated property with respect to $\varphi^{-1}(B_{\min \{ \epsilon ',\epsilon ''\}}(v_y)\cap P)$.
Furthermore, by the construction of $B_{\min {\{\epsilon',\epsilon''\}}}(v_y)$ for each $v_y \in \vertex (P)$, we have that $L_{\varphi}^{2}$ is strictly proper for $\prop [L_{\varphi}^{2}]$.
Thus, by Theorem \ref{thm:conspropcal},  $(L_{\varphi}^{2}, \maplink )$ is $\ell_{0-1}$-calibrated for at least the distributions $\P =\cup_{v_y\in\vertex (P)} \varphi^{-1}(B_{\min \{ \epsilon ',\epsilon ''\} }(v_y)\cap P) $ as well as $\varphi (\P)$ is a strict calibration region.
Furthermore, since $B_{\min \{\epsilon ',\epsilon ''\}}$ for each $v_y\in\vertex (P)$ is non-empty, we have that $ \P \neq\emptyset$.
\end{proof}

Although strict calibration regions $R_y$ exist for each outcome $y\in\Y$ via the polytope embedding, tightly characterizing strict calibration regions is non-trivial.
Since the level sets of elicitable properties are convex within the underlying simplex, characterizing the strict calibration regions becomes a collision detection problem, which is often computationally hard.







\section{Restoring Inconsistent Surrogates via Low-Noise Assumptions}\label{sec:lownoise}
Looking towards application, we refine our results on the existence of strict calibration regions by examining a low-noise assumption, which provides an interpretable calibration region (\S~\ref{subsec:calib-low-noise}). 
We show which low-noise assumptions imply calibration when embedding $2^d$ outcomes into $d$ dimensions and $d!$ outcomes into $d$ dimensions (\S~\ref{subsec:permutahedron}). 
We refer the reader to Appendix \ref{app:strict} for omitted proofs.

\subsection{Calibration via Low Noise Assumptions}\label{subsec:calib-low-noise}

We demonstrate that every polytope embedding leads to calibration under some low-noise assumption. 
Our results enable practictioners to choose the dimension $d$, unlike in previous works.
 Following previous work \citep{agarwal2015consistent}, we define a low noise assumption to be a subset of the probability simplex on the label distribution parameterized by $\hat \alpha$: $\Theta_{\hat{\alpha}} =\{p\in \Delta_{\Y}\mid \max_{y\in \Y}p_y \geq 1-\hat{\alpha} \}$ where $\hat{\alpha}\in [0,1]$.
This noise assumption can be understood as Massart noise \citep{massart2006risk} in the multiclass setting.

Given $\alpha \in [0,1]$ and $y\in \Y$, we define the set $\Psi^{y}_{\alpha}=\{ (1-\alpha )\delta_y +\alpha \delta_{\hat{y}} \mid \hat{y}\in \Y  \}$.
With an embedding $\varphi$ onto $P$, we define the set $P_{\alpha}^{y}:=\varphi (\conv (\Psi^{y}_{\alpha}))$, a scaled version of $P$ anchored at $v_y$, that moves vertices $(1-\alpha)$ proportionally towards $y$, (Figure~\ref{fig:noise-plots}R). 



\begin{theorem}\label{thm:lownoisexists}
 Let $L^2$ be a Square Loss, $\varphi$ be any polytope embedding, and $L^{2}_{\varphi}$ be the loss induced by $(L^2, \varphi )$.
There exists an $\alpha \in [0,.5)$ such that for the link 
$\maplinknoise (u) = \argmin_{y\in \Y   }  \|u-P^{y}_{\alpha}\|_{2},$
$(L^{2}_{\varphi},\maplinknoise )$ is $\ell_{0-1}$-calibrated over the distributions $\Theta_{\alpha} := \{p\in \Delta_{\Y}\mid \max_{y\in \Y}p_y \geq 1-\alpha\}$.  
\end{theorem}

\begin{proof} 
\noindent \textbf{Part 1 (Choosing $\alpha \in [0,.5 )$)}: 
Following the proof of Theorem \ref{thm:consvert}, there exists a strict calibration region $R$ via $(L^{2}_{\varphi},\maplinknoise)$ w.r.t. $\ell_{0-1}$ and hence, there exists an $\epsilon >0$ such that $B_{\epsilon}(v_y)\cap P\subseteq R_{y}$ for all $y\in\Y$ where .
Given that $\vertex (P)$ are unique points, there exists a sufficiently small $\epsilon' >0$ such that $B_{\epsilon'}(v) \cap B_{\epsilon'}(\hat{v})=\emptyset$ for all $v,\hat{v}\in\vertex (P)$ where $v\neq \hat{v}$.
Let $\epsilon'' =\min{(\epsilon ,\epsilon')}$.
For any $y\in \Y$, observe the set $\conv(\Psi_{\alpha}^{y})$, defined using any $\alpha \in [0,.5)$, is a scaled-down translated unit simplex and that for all $p\in \conv (\Psi_{\alpha}^{y}) \subset \Delta_{\Y}$ it holds that $y = \text{mode}(p)$.

We shall show that for some sufficiently small $\alpha \in [0,.5)$, $P_{\alpha}^{y}$ is a scaled down version of $P$ positioned at the respective vertex $v_y$.
Furthermore, we shall show that $P_{\alpha}^{y}\subset B_{\epsilon''}(v_y )\cap P\subseteq R_y$ for all $y\in\Y$.
Observe that by linearity of $\varphi$,

$$P_{\alpha}^{y}:=\varphi(\conv  (\Psi_{\alpha}^{y}))= \conv (\varphi(\{(1-\alpha) \delta_y + \alpha \delta_{\hat{y}} | \hat{y} \in \Y\})) = \conv (\{(1-\alpha) v_y+ \alpha v_{\hat{y}} | \hat{y} \in \Y\}) $$

and hence, $P_{\alpha}^{y}$ is a scaled version of $P$ positioned at $v_y$.
Hence for some sufficiently small $\alpha$, $(1-\alpha) v_y+ \alpha v_{\hat{y}} \in  B_{\epsilon''}(v_y)$ for all $\hat{y}$ and hence $P_{\alpha}^{y} \subseteq B_{\epsilon''}(v_y) \subseteq R_y$.
With said sufficiently small $\alpha$, define $\maplinknoise$ and the respective sets $\conv (\Psi^{y}_{\alpha})$ for each $y\in \Y$.
Using the previous $\alpha$, define the set $\Theta_{\alpha}$ as well. \newline 

\noindent \textbf{Part 2 (Showing Calibration)}:
Recall, by Proposition \ref{thm:embedmin}, for any $p\in \Delta_{\Y}$, $u=\varphi (p)$ minimizes the expected surrogate loss $\E_{\Y\sim p} [L_{\varphi}^2(u,Y)]$.
For any fixed $y\in \Y$, observe that $\conv \{(1-\alpha )\delta_y +\alpha \delta_{\hat{y}}\mid \hat{y}\in \Y \}=\{ p\in\Delta_{\Y}:p_y \geq 1-\alpha \}\subset \Delta_{\Y}$ and hence, by Proposition \ref{thm:embedmin}, $\cup_{y\in\Y}P_{\alpha}^{y} $ contains all of the minimizing surrogate reports with respect to $\Theta_{\alpha}$. 
Finally, since every $\cup_{y\in\Y}P_{\alpha}^{y} \subseteq R_{\Y}$, by the definition of strict calibration region, it holds that $(L_{\varphi}^{2}, \maplinknoise )$ is  $\ell_{0-1}$-calibrated for $\Theta_{\alpha}$.
\end{proof}



\subsection{Embedding into the Unit Cube and Permutahedron under Low-Noise }\label{subsec:permutahedron}

In this section, we demonstrate embedding onto the unit cube and the permutahedron \citep{blondel2020learning,seger2018investigation}.
We show that by embedding $2^d$ outcomes into a $d$ dimensional unit cube $P^{\square}$, $(L_{\varphi}^2, \psi^{P^{\square},\alpha})$ is calibrated over $\Theta_\alpha$ for all $\alpha \in [0,\frac 1 2)$.
Furthermore, we found that by embedding $d!$ outcomes into a $d$ dimensional permutahedron $P^{w}$, $(L_{\varphi}^2, \psi^{P^{w},\alpha})$ is calibrated for $\Theta_\alpha$ for $\alpha \in [ 0,\frac{1}{d})$. 
Theorem \ref{corr:lownoise} enables us to simultaneously study the aforementioned embeddings.

\begin{theorem}\label{corr:lownoise}
    Let $L^2$ be a Square Loss $\varphi$ be any polytope embedding, and $L^{2}_{\varphi}$ be the loss induced by $(L^2,\varphi )$.
    Fix $\alpha \in [0,.5)$ and with it define $\Theta_{\alpha}$.
    If for all $y,\hat{y}\in \Y$ such that $y\neq \hat{y}$ it holds that $P^{y}_{\alpha}\cap P^{\hat{y}}_{\alpha} =\emptyset $, then $(L^{2}_{\varphi}, \maplinknoise )$ is $\ell_{0-1}$-calibrated for  $\Theta_{\alpha}$ where $\maplinknoise (u) = \argmin_{y\in \Y   }  \|u-P^{y}_{\alpha}\|_{2}$. 
\end{theorem}\vspace{-.4cm}

\begin{proof}
Pick an $\alpha$ such that for all $y, \haty \in \Y$, $P^{y}_{\alpha}\cap P^{\hat{y}}_{\alpha} =\emptyset $.
Define $\Theta_{\alpha}$ and $\maplinknoise$ accordingly.
For $p\in \Theta_{\alpha}$ and some $y\in\Y$, say a sequence $\{u_m\}$ converges to $\prop [L_{\varphi}^{2}] (p) =\varphi (p)\in P^{y}_{\alpha}$, where the equality follows from Proposition \ref{thm:embedmin}.
Given that each $P^{y}_{\alpha}$ is closed and pairwise disjoint, there exists some $\hat{\epsilon}>0$ such that for all $y,\hat{y}\in\Y$ where $y\neq \hat{y}$, it also holds that $ (P^{y}_{\alpha}+B_{\hat{\epsilon}}) \cap (P^{\hat{y}}_{\alpha}+B_{\hat{\epsilon}})=\emptyset$ where $+$ denotes the Minkowski sum.
Since $\{u_m\}$ converges to $\varphi(p)$, there exists some $N\in \mathbb{N}$ such that for all $n\geq N$, $\|u_n -  \varphi (p)   \|_{2}< \hat{\epsilon} $.
By the definition of $\maplinknoise$, any $u_n$ where $n\geq N$ will be mapped to $y$, the correct unique report given that $\prop [L_{\varphi}^2](p)\in P^{y}_{\alpha}$.
Hence, $(\prop [L_{\varphi}^{2}],\maplinknoise )$ is $\ell_{0-1}$-calibrated property with respect to $\Theta_{\alpha}$.
Finally, since $L_{\varphi}^{2}$ is strictly proper for $\prop[L_{\varphi}^{2}]$, by Theorem \ref{thm:conspropcal}, we have that  $(L_{\varphi}^{2}, \maplinknoise )$ is $\ell_{0-1}$-calibrated for $\Theta_{\alpha}$. 
\end{proof}

\paragraph{Unit Cube}

Define a unit cube in $d$-dimensions by $P^{\square} := \conv (\{-1,1\}^{d})$. 
Binary encoding outcomes into the elements of $\{-1,1\}^d$ (the vertices of a unit cube) is a commonly used method in practice (e.g., \citep{seger2018investigation,yu2018lovasz}).
We show that calibration holds under a low noise assumption of $\Theta_{\alpha}$ when $\alpha < .5$.

\begin{restatable}{corollaryc}{unitcubenoise}
  \label{corr:unitcubenoise}
   Let $\varphi$ be an embedding from $2^d$ outcomes into the vertices of $P^{\square}$ in $d$-dimensions and define an induced loss $L^{2}_{\varphi}$. 
   Fix $\alpha \in [0,.5)$ and define $\Theta_{\alpha}$.
$(L^{2}_{\varphi},\psi^{P^{\square},\alpha})$ is $\ell_{0-1}$-calibrated for $\Theta_{\alpha}$.

\end{restatable}

Corollary \ref{corr:unitcubenoise} suggests that binary encoding is an appropriate methodology when one has a prior over the data that the mode of the label distribution $\Pr[Y \mid X = x]$ is greater than half for all $x \in \X$.
Interestingly, the bound of $\alpha$ is not dependent on the dimension of $d$. 
We now present a result for embedding outcomes into a factorially lower dimension via the permutahedron.
Intuitively, ranking can be recast as a multiclass classification problem, in which case the outcomes are orderings of the $d$ possible labels.



\paragraph{Permutahedron}
Let $\Sc_{d}$ express the set of permutations on $[d]$.
The permutahedron associated with a vector $w \in\reals^{d}$ is defined to be the convex hull of the permutations of the indices of $w$, i.e.,
$ P^{w} := \conv ( \{ \pi (w)  \mid \pi \in\Sc_{d} \} ) \subset \reals^d$.
The permutahedron may serve as an embedding from $d!$ outcomes into $d$-dimensions; it is a natural choice for embedding full rankings over $d$ items.

\begin{restatable}{corollaryc}{permanoise}
  \label{corr:permanoise}
   Let $\varphi$ be an embedding from $d!$ outcomes into the vertices of $P^{w}$ in $d$ dimensions such that $w=(0,\frac{1}{\beta d},\frac{2}{\beta d},\dots ,\frac{d-1}{\beta d})\in \reals^{d}$ where $\beta = \frac{d-1}{2}$. 
   Fix $\alpha \in [ 0,\frac{1}{d})$.
    Then $(L^{2}_{\varphi},\psi^{P^w,\alpha})$ is $\ell_{0-1}$-calibrated over $\Theta_{\alpha}$.
\end{restatable}

The calibration region in Corollary \ref{corr:permanoise} show that consistency in $\Theta_\alpha$ shrinks exponentially in $d$.
Unless one has a prior that the data follows some form of a power distribution, Corollary \ref{corr:permanoise} suggests not to factorially embed outcomes.






\begin{figure}[t]
\centering

\begin{minipage}[t]{0.05\textwidth}
	\centering
\end{minipage}
\hfill
\begin{minipage}[t]{0.35\textwidth}
	\centering
	\includegraphics[scale=0.12]{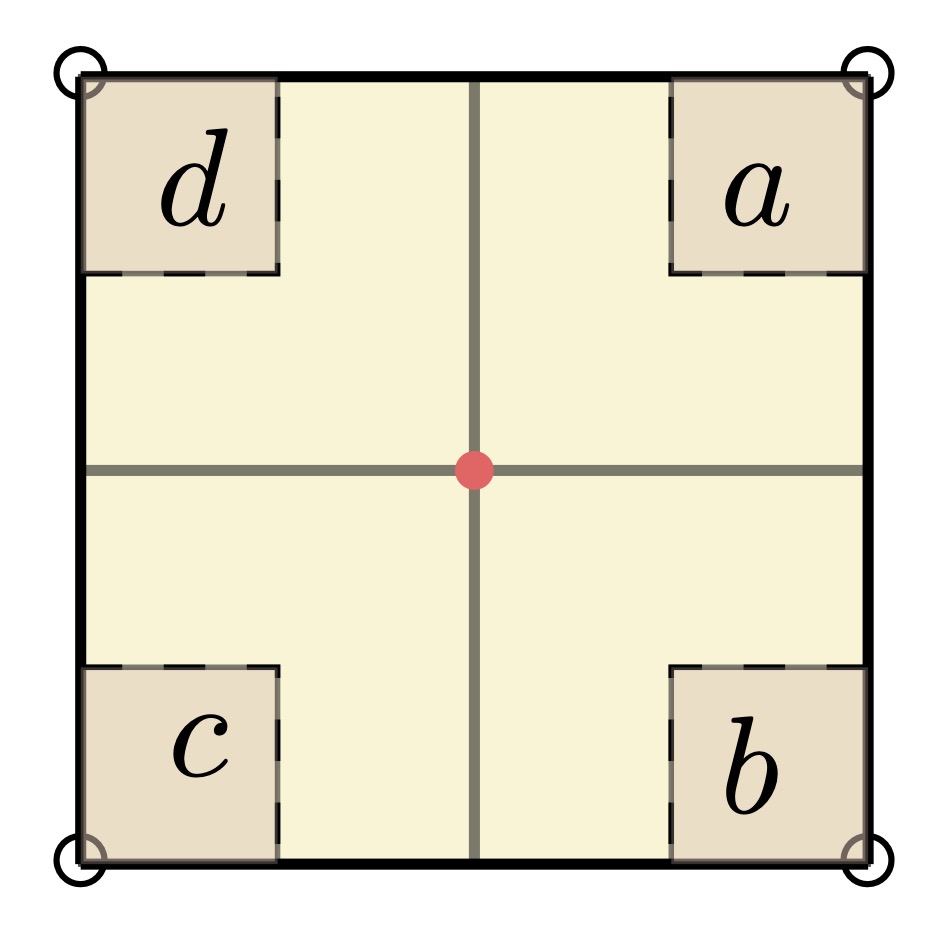}
\end{minipage}
\hfill
  \begin{minipage}[t]{0.45\textwidth}
    \includegraphics[scale=0.092]{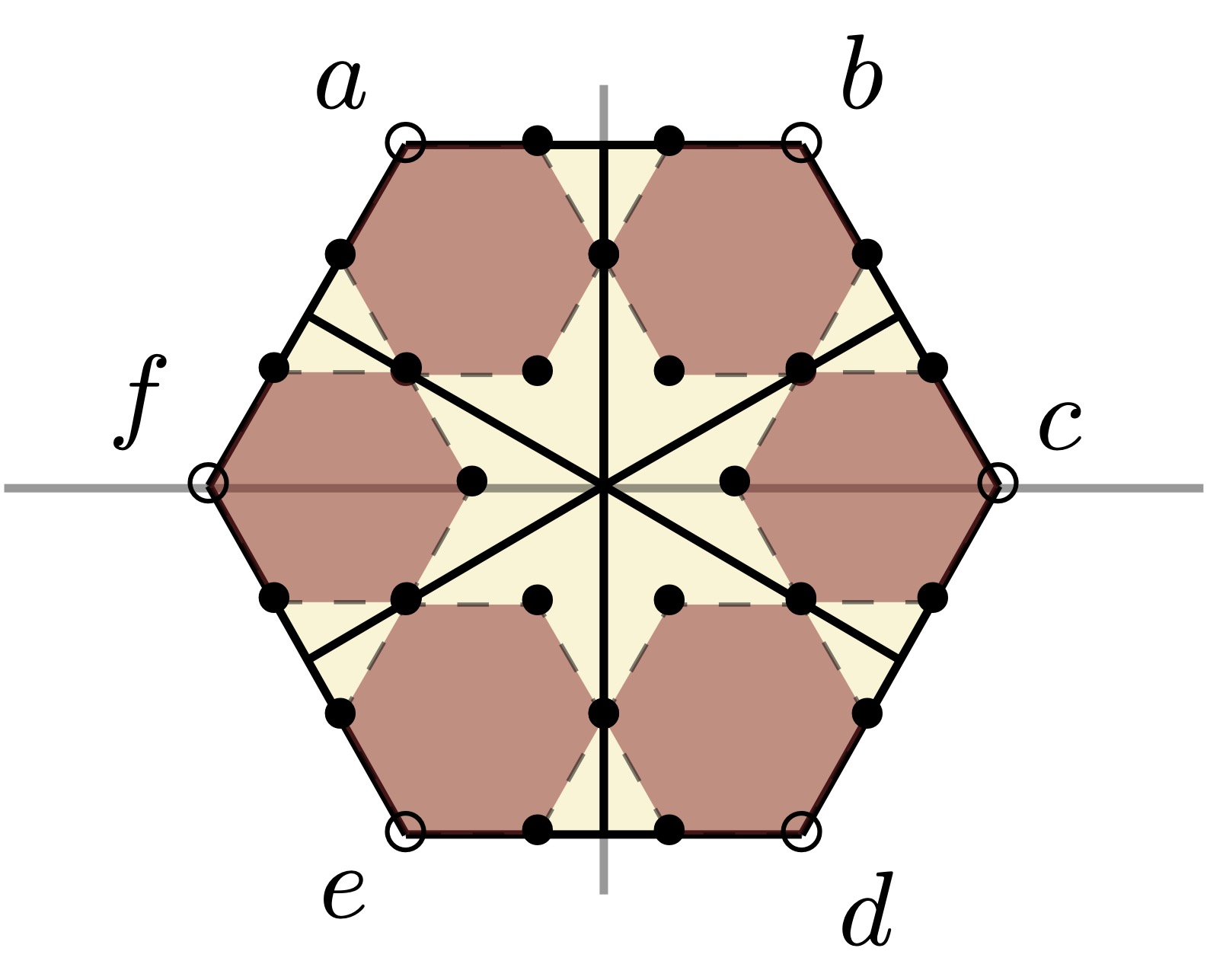}
  \end{minipage}
\begin{minipage}[t]{0.1\textwidth}
	\centering
\end{minipage}
\hfill
\caption{(Left) Corners represent the strict calibration regions for $\Theta_{\alpha}$ where $\Y =\{a,b,c,d \}$ is embedded into a two dimensional unit cube such that $\alpha = .25$. 
(Right) Auburn regions show that strict calibration holds for $\Theta_{\alpha}$ where $\Y =\{a,b,c,d,e,f\}$ is embedded into a three-dimensional permutahedron such that $\alpha =\frac{1}{3}-\epsilon $.}
\label{fig:noise-plots}
\end{figure}




\section{Elicitation in Low Dimensions with Multiple Problem Instances }\label{sec:multiinst}
The tools developed in previous sections now enable us to address the setting in which we require full consistency, $\P = \simplex$, but also desire surrogate prediction dimension $ d \ll n-1$.
We side-step the $n-1$ lower bound by utilizing multiple problem instances and aggregation of the outputs.  
Although cumulatively we have a larger surrogate prediction dimension than $n-1$, each individual problem instance has a less than $n-1$ surrogate prediction dimension. 
This approach is well-motivated since it allows for distributed computing of separate, smaller models which leads to faster convergence overall since in general optimization is at least $poly(d)$.
Previous work such as \cite{ramaswamy2014consistency} has explored the consistency of multiclass problem reductions; however, we take a different, geometrically motivated, approach.

\begin{definition}
Extending Definition \ref{def:elicit}, we say a loss and link pair $(L, \psi )$, where  $L: \reals^d \times \Y \to \reals$ and $\psi: \reals^d \to \R$, elicits a property $\Gamma:\P \toto \R$ on $\P\subseteq \Delta_{\Y}$ if  $\forall \; p \in \mathcal{P}, \; \Gamma (p)=\psi(\argmin_{u\in\reals^d}\E_{Y\sim p} [L(u,Y)])$.
\end{definition}

\begin{definition}[$(n,d,m)$-Polytope Elicitable]
Suppose we are given a property $\gamma :\P \toto \R$ such that $\P\subseteq \Delta_{\Y}$ and  $|\Y |=n$ finite outcomes.
Say we have $m$ unique polytope embeddings $\{\varphi_{j}:\simplex \to \reals^d \}_{j=1}^{m}$ where $d < n-1$, and a set of induced losses $\{L^{2}_{\varphi_{j}}\}_{j=1}^{m}$ and links $\psi_{j} :\reals^{d}\to \B_{j}$ defined wrt. $\varphi_{j}$, where $\B_j$ is an arbitrary report set.
For each $j\in [m]$, assume the pair $(L_{\varphi_{j}}^{2}, \psi_{j})$ elicits the property $\Gamma_j :\P \toto \B_j$.
If there exists a function $\Upsilon :\B_1\times \dots \times \B_m\toto \R$ such that for any $p\in\Delta_{\Y}$ it holds that $\Upsilon (\Gamma_1(p),\dots , \Gamma_m(p)) = \gamma (p)$,
we say that $\gamma$ is $(n,d,m)$-Polytope Elicitable over $\P$.
\end{definition}

\noindent Equivalently, we will also say that the pair $(\{ (L^{2}_{\varphi_j},\psi_{j}) \}_{j=1}^{m},\Upsilon)$ $(n,d,m)$-Polytope elicits the property $\gamma$ with respect to $\P$.

We shall express a $d$-cross polytope by $P^{\oplus} := \conv (\{ \pi ((\pm {1},0,\dots , 0)) \mid \pi\in \Sc_d\} )$ where $(\pm {1},0,\dots , 0)\in\reals^{d}$.
Observe that a $d$-cross polytope has $2d$ vertices.
For any vertex of a d-cross polytope $v\in\vertex (P^{\oplus})$, we shall say that $(v,-v)$ forms a diagonal vertex pair.

\begin{restatable}{lemmac}{compar}
  \label{lem:compar}
Say we are given a cross-polytope embedding $\varphi :\Delta_{2d}\to P^{\oplus}$ and induced loss $L^{2}_{\varphi}$.
Let $(v_{a_i}, v_{b_i})$, be the $i^{th}$ diagonal pair (i.e. $\varphi(\delta_{a_i}) = v_{a_i}$). Define the property $\Gamma^{\varphi}:\Delta_{2d} \to \B$ element-wise by 
$$\Gamma^{\varphi}(p)_i:=\left\{
    \begin{array}{lr}
        (<, a_i, b_i) &\text{if  } p_{a_i}<p_{b_i}\\
        (>, a_i, b_i) &\text{if  } p_{a_i}>p_{b_i}\\
        (=, a_i, b_i) &\text{if  }p_{a_i}=p_{b_i}.
    \end{array}
    \right.
$$ 
Furthermore define the link $\psi^{P^\oplus}:\reals^{d}\to \B$ with respect to each diagonal pair as
\begin{align*}
\psi(u; v_{a_i}, v_{b_i})_{i}^{P^{\oplus}} := \left\{
    \begin{array}{lr}
        (<, a_i, b_i) &\text{if  } ||u-v_{a_i} ||_{2} >  ||u-v_{b_i} ||_{2} \\
        (>, a_i, b_i) &\text{if  } ||u-v_{a_i} ||_{2}<  ||u-v_{b_i} ||_{2}\\
        (=, a_i, b_i) & \text{o.w.} 
    \end{array}
    \right.
\end{align*}
Then $(L_\varphi^{2}, \psi^{P^{\oplus}})$ elicits $\Gamma^{\varphi}$.
\end{restatable}

\begin{figure}[t]
\centering
  \begin{minipage}[t]{0.2\textwidth}
  \end{minipage}
  \hfill
\begin{minipage}[t]{0.3\textwidth}
	\centering
	\includegraphics[width=\textwidth]{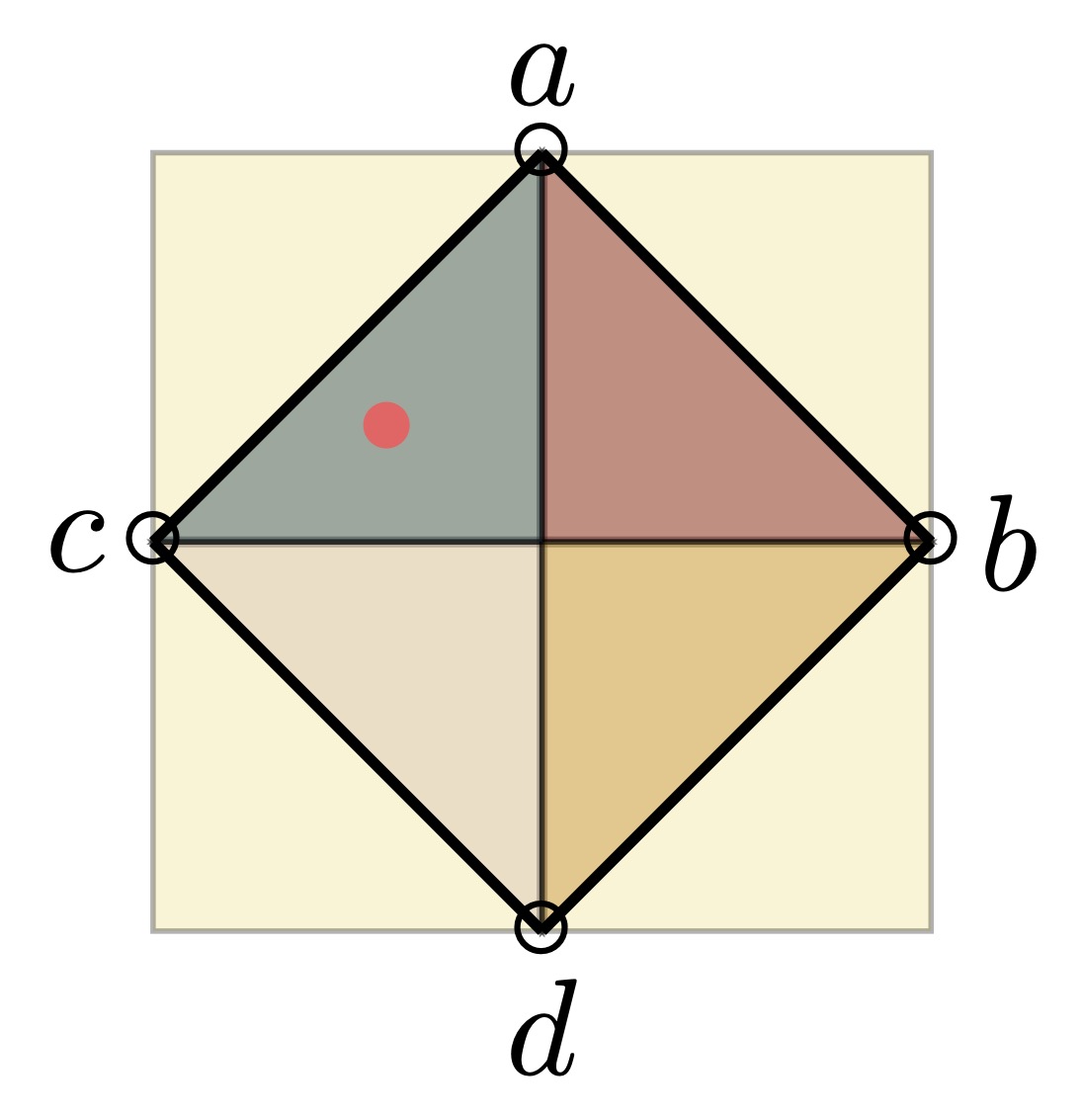}
\end{minipage}
\hfill
  \begin{minipage}[t]{0.3\textwidth}
    \includegraphics[width=\textwidth]{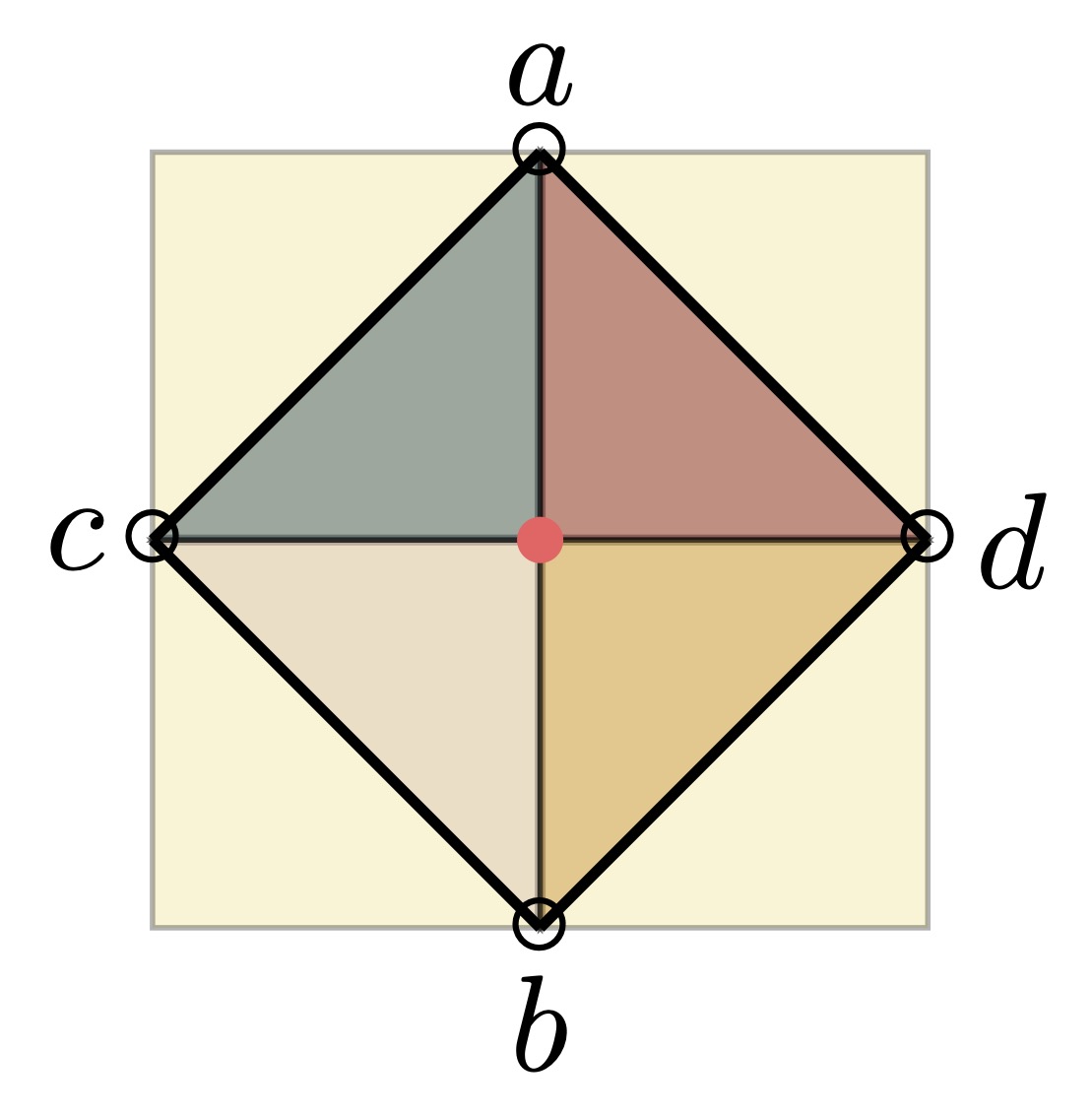}
  \end{minipage}
  \hfill
  \begin{minipage}[t]{0.2\textwidth}
  \end{minipage}
  \caption{Four outcomes embedded in $\reals^2$ in two different ways, with the minimizing reports \textcolor{red}{$\bullet$} for a distribution $p$." (Left) Configuration $\varphi_{1}$ with \textcolor{red}{$\bullet$} at $(-.5,.3)$ implying $p_a>p_d$ and $p_b>p_c$. (Right) Configuration $\varphi_{2}$ with \textcolor{red}{$\bullet$} at $(0,0)$ implying $p_a=p_b$ and $p_c=p_d$.  This implies the true distribution is $p = (0.4,0.4,0.1,0.1)$." } 
\label{fig:cross}
\end{figure}

The following theorem states that by using multiple problem instances, based on Lemma \ref{lem:compar}, we can Polytope-elicit the mode.
Algorithm \ref{alg:bsort} outlines how to aggregate the individual solutions to infer the mode. 
We defer the proof to Appendix \ref{app:strict}.
\begin{restatable}{theorem}{multiinstance}
    \label{thm:multi-instance}
    Let $d \geq 2$. The mode is $(2d, d, m)$-Polytope Elicitable for some $m \in [2d-1, d(2d-1)]$.
\end{restatable}

\begin{algorithm}
\caption{Elicit mode via comparisons and the $d$-Cross Polytopes}\label{alg:bsort}
\begin{algorithmic}
\Require $M = \{ (L^{2}_{\varphi_j},\psi^{P^{\oplus}}_j) \}_{j=1}^{m}$
\State Learn a model $h_j:\X\to \reals^{d}$ for each instance  $(L^{2}_{\varphi_j},\psi^{P^{\oplus}}_j) \in M$
\State For some fixed $x\in\X$, collect all $B_j \gets \psi^{P^{\oplus}}_j(h_j(x))$ where $B_j\in\B_j$ 
\State Report $ R \gets \text{FindMaxes\footnotemark[2]}( B_1,\dots , B_{m} )$
\end{algorithmic}
\end{algorithm}
\footnotetext[2]{Given all comparisons, a sorting algorithm can be used to compute the set of $r \in \R$ such that $p_r$ is maximum.}

Although Theorem \ref{thm:multi-instance} states that the mode is $(2d, d, m)$-Polytope Elicitable for some $m \in [2d - 1, d(2d -
1)]$, it does not state how we select said $\{ (L^{2}_{\varphi_j},\psi^{P^{\oplus}}_j) \}_{j=1}^{m}$ problem instances in an optimal manner.
Unfortunately, selecting the min number of problem instances reduces to a a minimum set cover problem which is computationally hard. 
Even so, through a greedy approach, one can choose problem instances that are log approximate optimal relative to the true best configuration.  
In practice using real data, given that these are asymptotic results, we may have conflicting logic for the provided individual reports.
In Appendix \ref{app:multprob}, we discuss an approach of how to address this in practice.


\section{Discussion and Conclusion}
This work examines various tradeoffs between surrogate loss dimension, restricting the region of consistency in the simplex when using the 0-1 loss, and number of problem instances.
Since our analysis is based on an embedding approach commonly used in practice, our work provides theoretical guidance for practitioners choosing an embedding. 
We see several possible future directions.
The first is a deeper investigation into hallucinations.
Future work could investigate the size of the hallucination region in theory, and the frequency of reports in the hallucination region in practice.
Another direction would be to construct a method that efficiently identifies the strict calibration regions and the distributions embedded into them.
This would provide better guidance on whether or not a particular polytope embedding aligns with one's prior over the data.
Another possible direction would be to explore whether concepts from this paper could be applied to the underlying problem of cost-sensitive multiclass classification. 
Finally, another direction is to identify other properties that can be elicited via multiple problem instances while also reducing the dimension of any one instance. 

{\bf Broader Impacts:} Our work broadly informs the selection of loss functions for machine learning.
Thus our work may influence practitioners’ choice of loss function. 
Of course, such loss functions can be used for ethical or unethical purposes. 
We do not know of particular risks of negative impacts of this work beyond risks of machine learning in general.






\subsection*{Acknowledgments}
We thank Rafael Frongillo for discussions about hallucinations, which led to the exploration of many of the ideas in
this work and Amzi Jeffs for discussions regarding convex geometry. This material is based upon work supported by
the National Science Foundation under Award No. 2202898 (JF).

\bibliographystyle{plainnat}
\bibliography{bib}

\newpage 
\appendix


\section{Notation tables} \label{app:notation}

\begin{table}[h]
	\centering
	\begin{tabular}{ll}
		Notation & Explanation \\
		\toprule
		$r\in\mathcal{R}$ & Prediction space \\
            $y\in\Y$ & Label space \\
            $\mathcal{P}\subseteq \Delta_{\Y}$ & Subset of simplex over $\Y$ \\
            $\delta_{y}$ & Point mass distribution on $y\in\Y$ \\ 
            $[d]:= \{1,\dots ,d \}$ & Index set \\
            $\ones_{S}\in \{0,1\}^{d}$ s.t. $(\ones_S)_i =1 \Leftrightarrow i\in S$ & 0-1 Indicator on set $S\subseteq [d]$ \\
            $\C\subset \reals^{d}$ & Closed convex set \\
            $u\in\reals^{d}$ & Surrogate prediction space \\ 
            $\text{Proj}_{\C}(u) := \argmin_{x\in \C} ||u-x||_{2}$ & Projection onto closed convex set \\   
            $\pi \in \Sc_{d}$ & Permutations of $[d]$ \\
		$\ell :\R\times \Y \to \reals_{+}$  & Discrete loss \\
            $L :\reals^{d}\times \Y \to \reals$  & Surrogate loss \\
            $\psi :\reals^{d}\to \Y$ & Link function \\ 
             $\E_{Y\sim p} [\ell (r,Y)]$ & Expected discrete loss \\ 
             $\E_{Y\sim p} [L (u,Y)] $ & Expected surrogate loss \\ 
             $\Gamma :\mathcal{P}  \to 2^\mathcal{R}\setminus \{\emptyset \}$ & Property \\ 
             $\Gamma_r := \{ p\in\mathcal{P} \mid r = \Gamma (p) \}$ & Level set of property \\ 
             $\prop [L]$ & Property elicited by $L$ \\
             $\ell_{0-1 }:\Y \times \Y \to \{ 0,1  \}$ & Zero-one loss \\
             $\modeprop :\Delta_{\Y}\to 2^{\Y}\setminus \{\emptyset\}$ & Mode property \\ 
		\bottomrule
	\end{tabular}
\caption{Table of general notation}\label{tab:notation}
\end{table}

\begin{table}[h]
	\centering
	\begin{tabular}{ll}
		Notation & Explanation \\
		\toprule
		$P\subset \reals^{d}$ & Polytope \\
            $\vertex (P)$ & Vertex Set \\ 
            $E(P)$ & Edge set of $P$ \\ 
            $\varphi: \simplex \to P$ & Polytope Embedding \\ 
            $P^{\square} := \conv (\{-1,1\}^{d})$ & Unit cube \\ 
            $P^{w} := \conv ( \{ \pi (w)  \mid \pi \in\Sc_{d} \} )$ s.t. $w\in\reals^{d}$ & Permutahedron \\ 
            $P^{\oplus} := \conv (\{ \pi ((\pm {1},0,\dots , 0)) \mid \pi\in \Sc_d\} )$ & Cross polytope \\ 
            $L^{2}:\reals^{d}\times \Y \to \reals_{+}$ & Square Loss \\ 
            $L_{\varphi}^{2}: \reals^{d}\times \Y \to \reals_{+}$ & $(L^2 ,\varphi )$ Induced Loss \\
            $\maplink  : \reals^{d} \to \Y$ & MAP Link \\ 
		\bottomrule
	\end{tabular}
\caption{Table of polytope and embedding notation}\label{tab:polynotation}
\end{table}

\begin{table}[h]
	\centering
	\begin{tabular}{ll}
		Notation & Explanation \\
		\toprule
		$\H \subseteq P$ & Hallucination region \\
            $R\subset P$ & Strict calibrated region \\ 
            $R_y\coloneqq R \cap \psi_y $ & Intersection of link level set and strict cal. region \\ 
            $R_{\Y} :=\cup_{y\in\Y}R_y$ & Union of $R_y$ \\
            $\Theta_{\alpha}\subset \Delta_{\Y}$ & Low-noise assumption \\ 
            $\Psi^{y}_{\alpha}=\{ (1-\alpha )\delta_y +\alpha \delta_{\hat{y}} \mid \hat{y}\in \Y  \}$ & Scaled vertex set \\ 
            $P_{\alpha}^{y}:=\varphi (\conv (\Psi^{y}_{\alpha}))$ & Scaled version of $P$ anchored at $v_y$\\
            
		\bottomrule
	\end{tabular}
\caption{Table of calibration region notation}\label{tab:calnotation}
\end{table}


\newpage
\section{Polytopes, Omitted Proofs, and Results}\label{app:strict}

\subsection{Polytopes}\label{app:polyapp}
A Convex Polytope $P\subset \reals^{d}$, or simply a polytope, is the convex hull of a finite number of points $u_1,\dots ,u_n\in\reals^{d}$.
An extreme point of a convex set $A$, is a point $u\in A$ such that if $ u = \lambda y + (1 - \lambda  )z$ with $y,z \in A$ and $\lambda \in [0, 1]$, then $y = u$ and/or $z = u$. 
We shall denote by $\vertex (P)$ a polytope's set of extreme points.
A polytope can be expressed by the convex hull of its extreme points, i.e. $P=\conv (\vertex (P) )$ \citep[Theorem 7.2]{brondsted2012introduction}.

We define the dimension of $P$ via $\dim (P) :=\dim (\affhull  (P))$ where $\affhull  (P)$ denotes the smallest affine set containing $P$. 
A set $F\subseteq P$ is a face of $P$ is there exists a hyperplane $H(y,\alpha ) :=\{ u\in\reals^{d}\mid \langle u,y \rangle = \alpha  \}$ such that $F=P\cap H$ and $P\subseteq H^+$ such that $H^{+}(y,\alpha ) :=\{ u\in\reals^{d}\mid \langle u,y \rangle \leq \alpha  \}$.
Let $F_i(P)$ where $i\in [d-1]$ denote set of faces of dim $i$ of a polytope $P$.
A face of dimension zero is called a vertex and a face of dimension one is called an edge.
We define the edge set of a polytope $P$ by $E(P):=\{ \conv ( (v_i,v_j))  \mid (v_i,v_j) \subseteq  \binom{\vertex (P)}{2} , \conv  ((v_i,v_j )) \in F_{1}(P)  \} $.
We define the neighbors of a vertex $v$ by $\neigh (v ;P) := \{\hat{v}\in \vertex (P) \mid \; \conv (( v,\hat{v} ))  \in E(P)   \}$.
We will denote $\conv ( (v,\hat{v}))  \in E(P)$ by as $e_{v,\hat{v}}$ and $\neigh (v ;P)$ by $\neigh (v)$ when clear from context.

\subsection{Omitted Proofs from \S~\ref{sec:calregions}}

\embedmin*
\begin{proof}
For any fixed $p\in\Delta_{\mathcal{Y}}$, observe solving for the minimizer of $\E_{Y\sim p} [L_{\varphi}^{2}(u,Y) ] $ we get
$$\frac{d}{du_i} \E_{Y\sim p} [L_{\varphi}^{2}(u,Y) ] = \frac{d}{du_i} \sum_{y\in\mathcal{Y}} \frac{p_y}{c} \|u -v_y\|^{2}_{2}+f(y)= \frac{2}{c} \sum_{y\in\mathcal{Y}}p_y(u_i-v_{y,i}) $$
$$\Rightarrow \nabla_u \E_{Y\sim p} [L_{\varphi}^{2}(u,Y) ]=  \frac{2}{c}\sum_{y\in\mathcal{Y}}  p_y u- \frac{2}{c}\sum_{y\in\mathcal{Y}}  p_yv_y = 0 $$
$$u^* = \sum_{y\in\mathcal{Y}}  p_yv_y =  \varphi (p)~.$$
Thus, by the construction of the polytope embedding, it holds that $u^* = \varphi (p)$.
Since Square Loss are strictly convex functions, $u^*$ uniquely minimizes $\E_{Y\sim p}[L_{\varphi}^{2}(u,Y)]$.

Conversely, every $\hat{u}\in P$ is expressible as a convex combination of vertices; hence, by the definition of $\varphi$, for some distribution, say $\hat{p}\in \Delta_{\Y}$, it holds $\hat{u} = \varphi (\hat{p})$.
Therefore, it holds that $\hat{u}$ minimizes $\E_{Y\sim \hat{p}}[L_{\varphi}^{2}(u,Y)]$.
\end{proof}


\halreal* 

\begin{proof}
Choose a $y\in\Y$.
We abuse notation and write $\vertex (P)\setminus v_y := \vertex (P)\setminus \{v_y\}$.
Observe all $u\in \conv (\vertex (P)\setminus v_y )\cap \maplink_{y}$ can be expressed as a convex combination of vertices without needing vertex $v_{y}$.
The coefficients of said convex combination express a $p\in\Delta_{\Y}$ that is embedded to the point $u\in P$ where $p_y =0$.
Yet, by Proposition \ref{thm:embedmin}, said $u$ is an expected minimizer of $L^{2}_{\varphi}$ with respect to $p$.
Given the intersection with $\maplink_{y}$ and by Definition \ref{def:hallucination}, it holds that $\cup_{y\in\Y } \conv (\vertex (P)\setminus v_y )\cap \maplink_{y}\subseteq \H$.

We now shall show that $\H \subseteq \cup_{y\in\Y } \conv (\vertex (P)\setminus v_y )\cap \maplink_{y}$.
Fix $y\in\Y$.
Assume there exists a point $u\notin \conv (\vertex (P)\setminus v_y)\cap \maplink_{y}$ such that there exists some $p\in\Delta_{\Y}$ where $\varphi (p)=u$, $p_y =0$, and $\maplink (u)=y$.
Since $\maplink (u)=y$ and $u\notin \conv (\vertex (P)\setminus v_y))\cap \maplink_{y}$, it must be the case that $u\notin \conv (\vertex (P)\setminus v_y)$.
However, that implies that $u$ is strictly in the vertex figure and thus must have weight on the coefficient for $y$.
Thus, forming a contradiction that $p_y=0$ which implies that $\H =\cup_{y\in\Y } \conv (\vertex (P)\setminus v_y )\cap \maplink_{y}$.

To show non-emptiness of $\H$, we shall use Helly's Theorem (\citet{rockafellar1997convex}, Corollary 21.3.2).
W.l.o.g, assign an index such that $\Y =\{ y_{1},\dots ,y_{d},y_{d+1},\dots ,y_n \}$.
Observe the elements of the set $\{ \Y \setminus y_i  \}_{i=1}^{n}$ each differ by one element. 
W.l.o.g, pick the first $d+1$ elements of the previous set.
Observe $|\cap_{i=1}^{d+1}\Y\setminus y_i |=|\Y\setminus \{y_1,\dots ,y_d,y_{d+1}\}|=n-(d+1)>0$. 
Hence, by Helly's theorem and uniqueness of $y_i$'s, $\cap_{y\in\Y} \conv (\vertex (P) \setminus v_y) \neq \emptyset$.

Pick a point $u' \in \cap_{y\in\Y} \conv (\vertex (P) \setminus v_y) $.
Since $\maplink$ is well-defined, $u' $ will be linked to some outcome  $y'\in\Y$ and thus $u'\in \conv (\vertex  (P)\setminus v_{y'} )\cap \maplink_{y'}\subset \H$.
 Yet, $u' $ can be expressed as a convex combination which does not use $v_{y'}$ since it lies in $\cap_{y\in\Y} \conv (\vertex (P) \setminus v_y )$.
Thus, by using Proposition \ref{thm:embedmin} and by the definition of Hallucination (Def. \ref{def:hallucination}), we have that $\H\neq \emptyset$. 
\end{proof}


\begin{lemma}[Proposition 1.2.4]\citep{hiriart2004fundamentals}\label{lem:polytopoly}
 If $\varphi$ is an affine transformation of $\reals^n$ and $A\subset \reals^n$ is convex, then then the image $\varphi (A)$
is also convex. In particular, if the set $A$ is a convex polytope, the image is also a convex polytope.   
\end{lemma}

\begin{lemma}
Let $L^2$ be a Square Loss, $\varphi$ be any polytope embedding, $\psi$ be the MAP link, and $L_{\varphi}^{2}$ be the loss induced by $(L^2, \varphi)$.
Assume the target loss is $\ell_{0-1}$. 
    If a point is in a strict calibrated region such that $u\in R_y$ for some $y\in\Y$, it is necessary that $u \in \conv (\{ v_y\}\cup \neigh (v_y))\setminus \conv (\neigh (v_y))$.
\end{lemma}

\begin{proof}
    If $u\in R_y$ and $u\in P\setminus \big( \conv (\{ v_y\}\cup \neigh (v_y))\setminus \conv (\neigh (v_y))\big)$, then $u$ can be expressed as a convex combination which has no weight on the coefficient for $v_y$.
    Hence, there exists a distribution embedded into $u$ where $y$ would not be the mode, thus violating the initial claim that $u\in R_y$.
\end{proof}


\begin{lemma}\label{lemma:bdpw}
Let $L^2$ be a Square Loss, $\varphi$ be any polytope embedding, $\psi$ be the MAP link, and $L_{\varphi}^{2}$ be the loss induced by $(L^2, \varphi)$.
For any $u\in e_{(v_i,v_j)}\in E(P)$, it holds that  $|\varphi^{-1}(u)|=1$.
\end{lemma}

\begin{proof}
Observe, the two vertices of an edge define the convex hull making up the edge and hence,  by (\cite{gruber2007convex} ,Theorem 2.3) the two vertices are affinely independent.
Therefore, all elements of the edge have a unique convex combination which are expressed by the convex combinations of the edge's vertices.
Given the relation of the embedding $\varphi$ and convex combinations of vertices expressing distributions, it holds that $|\varphi^{-1}(u)|=1$.
\end{proof}


\begin{lemma}\label{lem:samedim}
  Let $L^2$ be a Square Loss, $\varphi$ be a polytope embedding, and $L^{2}_{\varphi}$ be the induced loss by $(L^2,\varphi)$.
  For all $y\in\Y$, it holds that $\dim (\varphi (\text{mode}_y))=\dim (P)\geq 2$.
\end{lemma}
\begin{proof}
By the construction of $\varphi$, we know that $\dim (P) \geq 2$.
Fix $y\in\Y$.
By Lemma \ref{lemma:bdpw}, we know that any edge connected from $v_y$ and $\hat{v}\in \neigh (v_y)$, the distributions embedded into the half of the line segment closer to $v_y$, $y$ is in the mode.
By Lemma \ref{lem:polytopoly}, we know that $\varphi (\modeprop_y)$ is a convex set.
Thus, the convex hull of the half line segments is part of $\varphi (\modeprop_y)$.
Since each vertex has at least $\dim (P)$ neighbors, it holds that $\dim (\varphi (\modeprop_y))=\dim (P)$.
\end{proof}



\subsection{Omitted Proofs from \S~\ref{sec:lownoise}}


\unitcubenoise*
\begin{proof}
W.l.o.g, say the outcome $y_1 \in\Y $ is embedded into $\ones_{[d]}\in \vertex (P^{\square})$.
Say $\alpha = .5$.
Observe that

\[
\Psi^{y_1}_{\alpha} = \left\{
\begin{pmatrix} 1\\ 0\\ \vdots \\ 0 \\ 0 \end{pmatrix},
\begin{pmatrix} 1-\alpha \\ \alpha\\ \vdots \\ 0 \\ 0 \end{pmatrix},
\begin{pmatrix} 1-\alpha\\ 0\\ \alpha\\ \vdots \\ 0 \end{pmatrix},
\dots , 
\begin{pmatrix} 1-\alpha\\ 0 \\ \vdots \\ \alpha \\ 0 \end{pmatrix},
\begin{pmatrix} 1-\alpha\\ 0\\ \vdots \\ 0 \\ \alpha \end{pmatrix}
\right\}
\]

\noindent and that $1\geq (1-\alpha)\pm \alpha \geq 0$ for any $\alpha \in (0,.5)$. 
Hence, for any $\alpha \in (0,.5)$ it holds that $P^{y_1}_{0.5}=\conv (\{0,1\}^{d})$ and furthermore $P_{\alpha}^{y_1}\subset P^{y_1}_{0.5} \subset \reals^{d}_{>0}$.
By symmetry of $P^{\square}$ and the linearity of $\varphi$, for any $\alpha \in (0,.5)$ and $y\in\Y$, we have that $P_{\alpha}^{y}$ is a strict subset of the orthant that contains $v_y$.
Hence, for all $y,\hat{y}\in \Y$ such that $y\neq \hat{y}$, it holds that $P^{y}_{\alpha}\cap P^{\hat{y}}_{\alpha}=\emptyset $.
Thus by Theorem \ref{corr:lownoise}, $(L^{2}_{\varphi},\psi^{P^{\square},\alpha})$ is $\ell_{0-1}$-calibrated for $\Theta_{\alpha}$ where $\alpha \in (0,.5)$.
\end{proof}


\permanoise*
\begin{proof}
Let $\Delta_{d}:= \conv (\{\ones_{[i]} \in \reals^d \mid i\in [d] \})$ and observe $P^{w}\subset \Delta_{d}$ since for all $\pi$, $\|\pi \cdot w \|_1 = \|w \|_1 = 1$.
Observe that $P^w$ can be symmetrically partitioned into $d!$ regions with disjoint interiors, one for each permutation $\pi \in \Sc_{d}$ via $\Delta_{d}^{\pi}:= \{u\in\Delta_d \mid u_{1}\leq \dots \leq u_{d} \}$.
Fix $\pi \in \Sc_d$ and w.l.o.g assume $\pi$ is associated with the constraints $\Delta_{w}^{\pi}:= \{u\in\Delta_w \mid u_{1}\leq \dots \leq u_{d} \}$ implying that $\pi (w) = (\frac{0}{\beta d},\frac{1}{\beta d},\dots ,\frac{d-1}{\beta d})$. 
Let $\alpha =\frac{1}{d}$ and define $\Theta_{\alpha}$.
With respect to $\Theta_{\alpha}$, let $y:=\varphi^{-1}(\pi (w))\in\Y$ and $\hat{y}:=\varphi^{-1}(\hat{\pi}(w))\in\Y$ such that $\hat{\pi}\in\Sc_d$.
Thus the set $\Psi^{y}_{\alpha}:= \{(1-\frac{1}{d})\delta_y +(\frac{1}{d})\delta_{\hat{y}}\mid \hat{y}\in\Y \}$ is mapped via $\varphi$ to the following points 
$$\varphi (\Psi^{y}_{\alpha}) = \{ (1-\frac{1}{d} ) (\pi ( w) )+ (\frac{1}{d}) (\hat{\pi} ( w))\mid \hat{\pi}\in\Sc_d \}$$
within the permutahedron. 

We shall show that $P^{y}_{\alpha}\subseteq \Delta_{d}^{\pi}$.
If this were not true, there would exists an element of $w^{\pi,\hat{\pi}} \in \varphi (\Psi^{y}_{\alpha})$ such such that for some pair of adjacent indices, say $i,i+1\in [d-1]$, $w^{\pi,\hat{\pi}}_{i}> w^{\pi,\hat{\pi}}_{i+1}$.
For sake of contradiction, fix $i\in [d-1]$ and assume there exists a $\hat{\pi}\in \Sc_d$ such that $w^{\pi,\hat{\pi}}_{i}> w^{\pi,\hat{\pi}}_{i+1}$.
Observe that any element of $\hat{\pi}(w)$ can be expressed by $\frac{j}{\beta d}$ using some $j \in \{0,1,\dots ,d-1\}$.
Thus,
\begin{align*}
&  w^{\pi,\hat{\pi}}_{i}> w^{\pi,\hat{\pi}}_{i+1}  &\\
&  \Leftrightarrow  (1-\frac{1}{d})(\frac{i-1}{\beta d}) +  (\frac{1}{d})(\hat{\pi}( w))_j > (1-\frac{1}{d})(\frac{i}{\beta d}) +  (\frac{1}{d})(\hat{\pi}( w))_{\hat{j}}  & \\
& \Rightarrow \quad (1-\frac{1}{d})(\frac{i-1}{\beta d})+(\frac{1}{d})(\frac{j}{\beta d}) >  (1-\frac{1}{d})(\frac{i}{\beta d})+(\frac{1}{d})(\frac{\hat{j}}{\beta d})  &   \\
& \Rightarrow\quad (i-1)(1-\frac{1}{d})+j(\frac{1}{d}) > i(1-\frac{1}{d})+\hat{j}(\frac{1}{d}) & \quad  \text{Multiply by $\beta d$} \\
& \Rightarrow \quad	1-d >  \hat{j}-j & 
\end{align*}
for some $j,\hat{j} \in \{0,1,\dots ,d-1\}$ where $j\neq \hat{j}$.\newline 

\noindent \textbf{Case 1}: ($j<\hat{j}$): The smallest value possible for $\hat{j}-j$ is $0-(d-1)$ however, $1-d  \ngtr 1-d$.\newline 

\noindent \textbf{Case 2}:($j>\hat{j}$): The smallest value possible for $\hat{j}-j$ is $1$ however, $1-d \ngtr 1$.\newline 


\noindent Hence, $P^{y}_{\alpha}\subseteq \Delta_{d}^{\pi}$ and specifically, there can exists an extreme point of $P^{y}_{\alpha}$ that lies on the boundary of $\Delta_{d}^{\pi}$ as shown in \textbf{Case 1}.
However, if $\alpha \in (0,\frac{1}{d})$, every extreme point of $P^{y}_{\alpha}$ moves closer to $\pi ( w)$ (besides the extreme point itself already on $\pi ( w)$) and therefore $ P^{y}_{\alpha}$ lies strictly within $\Delta_{d}^{\pi}$.
By symmetry of $P^{w}$ and the linearity of $\varphi$, this would imply that for all $y',y''\in \Y$ such that $y'\neq y''$ it holds that $P^{y'}_{\alpha}\cap P^{y''}_{\alpha}=\emptyset $.
Thus by Theorem \ref{corr:lownoise}, $(L^{2}_{\varphi},\psi^{P^{w},\alpha} )$ is $\ell_{0-1}$-calibrated for $\Theta_{\alpha}$ where $\alpha \in (0,\frac{1}{d})$.
\end{proof}


\subsection{Omitted Proofs from \S~\ref{sec:multiinst}}

\compar* 

\begin{proof}
W.l.o.g, fix a diagonal pair $(v_a,v_b)$ and let $v_a :=\ones_{[1]}$ and $v_b :=-\ones_{[1]}$. 
Define the embedding $\varphi$ accordingly.
We will show that the following is true for all distributions mapped via $\varphi$ to $u\in P^{\oplus}$.
\begin{align*} 
    &||u-v_a ||_{2} > ||u-v_b ||_{2} \iff p_a<p_b\\
    \text{OR  } &||u-v_a ||_{2} < ||u-v_b ||_{2} \iff p_a>p_b \\
    \text{OR  } &||u-v_a ||_{2} = ||u-v_b ||_{2} \iff p_a=p_b.
\end{align*}

First, fix $p\in\Delta_{2d}$. Recall, by Proposition \ref{thm:embedmin}, the minimizing report for $L^{2}_{\varphi}$ in expectation is $u=\varphi (p) \in P \subset \reals^{d}$.
We will prove the forward direction of the first and second lines. Then the reverse directions follow from the contrapositives.
\newline 

\noindent \textbf{Case 1}, $\implies$: Assume for contradiction that $p_a < p_b$ and $||\varphi (p) - v_a ||_2 < ||\varphi (p) -v_b||_2$. 
Then
\begin{align*}
  \quad \langle \varphi (p)-\ones_{1}, \varphi (p)-\ones_{1} \rangle  <&   \langle \varphi (p)+\ones_{1}, \varphi (p)+\ones_{1} \rangle    \\
(u_1-1)^{2}+\sum_{i=1}u_{i}^{2} < & (u_1+1)^{2}+\sum_{i=1}u_{i}^{2} \\
-u_1 < & u_1 ~.
\end{align*}
By the definition of a $d$-cross polytope $P^{\oplus} := \conv (\{ \pi ((\pm {1},0,\dots , 0)) \mid \pi\in \Sc_d\} )$ and the orthogonal relation between vertices, to express a $u\in P^{\oplus}$ as a convex combination of vertices, each diagonal pair of vertices coefficients solely influence the position along a single unit basis vector.
Hence, due to the definition of $\varphi$, we have $u_1 = \ones_{[1]}\cdot p_a -\ones_{[1]}\cdot p_b  <0 $ since we have assumed that $p_a < p_b$.
Hence $-u_1 <  u_1 < 0$, a contradiction.
\newline 

\noindent \textbf{Case 2}, $\implies$: Assume $p_a > p_b$ and $||\varphi (p) - v_a ||_2 < ||\varphi (p) -v_b||_2$.
By symmetry with case 1, all the inequalities are reversed, leading to the contradiction that $-u_1 > u_1 > 0$. 
\newline

\noindent \textbf{Case 3}: ($p_a = p_b$): Follows from the if and only ifs of cases 1 and 2. 

Hence $(L_\varphi^2, \psi_{\varphi})$ elicits $\Gamma^{\varphi}$.

\end{proof}


\multiinstance*

\begin{proof}
We will elicit the mode via the intermediate properties, $\Gamma^{\varphi_j}$, defined in Lemma \ref{lem:compar}.
First we construct a set of embeddings so that we guarantee that all the $\varphi_j$'s allow comparison between any pair of outcome probabilities.
For example, for each unique pair $(a,b)_j \in {\Y \choose 2}$ define an embedding: $\varphi_j(\delta_a) = \ones_{[1]}$ and $\varphi_j(\delta_b) = -\ones_{[1]}$, and embed every other remaining report $r \in \Y \setminus \{a,b\}$ arbitrarily.
Since $(L_\varphi^2, \psi^{P^{\oplus}})$ elicits $\Gamma^{\varphi}$, minimizing each $L^2_{\varphi_j}$ with a separate model yields us comparisons via the link $\psi^{P^{\oplus}}$. 
To find the set $r \in \R$ such that $p_r$ is maximum, we use a sorting algorithm that uses pairwise comparisons, such as bubble sort.
Hence with $\Upsilon$ as Algorithm \ref{alg:bsort}, we have that $\Upsilon(\{L^2_{\varphi_j}, \psi^{P^{\oplus}} \}) = \text{mode}(p)$.

Assuming there exist $\varphi
_j$s such that there is no redundancy in comparison pairs between each $\Gamma^{\varphi_j}$, we would need only $ \frac{d(2d-1)}{d} = 2d-1$ problem instances.
Hence, we establish our lower bound on the needed number of problem instances.
\end{proof}


\section{Hamming Loss Hallucination Example}\label{app:hal}

Hamming loss $\ell : \Y\times \Y \to \reals_{+}$ is defined by $\ell (y,\hat{y}) = \sum_{i=1}^{d}\ones_{y_i \neq \hat{y_i}}$ where $\Y =\{-1,1\}^d$.
Suppose $d=3$ and we have the following indexing over outcomes 
\begin{multline*}
\Y :=\{ y_1 \equiv (1,1,1), y_2 \equiv (1,1,-1), y_3 \equiv (1,-1,1), y_4 \equiv (-1,1,1), \\ y_5 \equiv (-1,-1,1), y_6 \equiv (1,-1,-1), y_7 \equiv (-1,1,-1), y_8 \equiv (-1,-1,-1)  \}~.
\end{multline*}

\noindent Let us define the following distribution  $$p_{\epsilon} =(0,\frac{1}{3}-\epsilon ,\frac{1}{3}-\epsilon,\frac{1}{3}-\epsilon,0,0,0,3\epsilon)\in\Delta_{\Y}$$ such that $\epsilon >0$.

\begin{itemize}
    \item $\E_{Y\sim p_{\epsilon}} [\ell (y_1,Y)] =1+6\epsilon$
    \item $\E_{Y\sim p_{\epsilon}} [\ell (y_2,Y)]=\E_{Y\sim p_{\epsilon}} [\ell (y_3,Y)]=\E_{Y\sim p_{\epsilon}} [\ell (y_4,Y)]= \frac{4}{3}+2\epsilon$
    \item $\E_{Y\sim p_{\epsilon}} [\ell (y_5,Y)]=\E_{Y\sim p_{\epsilon}} [\ell (y_6,Y)]=\E_{Y\sim p_{\epsilon}} [\ell (y_7,Y)]=\frac{7}{3}-4\epsilon$
    \item  $\E_{Y\sim p_{\epsilon}} [\ell (y_8,Y)]= 2-6\epsilon $
\end{itemize}
For all $\epsilon \in [0,\frac{1}{12})$, the minimizing report in expectation is $y_1=(1,1,1)$.
However, $p_{\epsilon,1}=0$ and thus, a hallucination would occur under a calibrated surrogate and link pair.

\section{Linking under Multiple Problem Instances}\label{app:multprob}

As stated in \S~\ref{sec:multiinst}, when using real data, given that these are asymptotic results, we may have conflicting logic for the provided individual reports.
In this section, we provide an approach such that the algorithm still reports information in the aforementioned scenario and will reduce to Algorithm \ref{alg:bsort} asymptotically. 
We build a binary relation table $M \in \{0,1\}^{n\times n}$ with the provided reports. 
Based on $M$, we select a largest subset of $S\subseteq \Y$ such that when $M$ is restricted to rows and columns corresponding to the elements of $S$, denoted by $M_S$, we have that $M_S$ is reflexive, antisymmetric, transitive, and strongly connected implying $M_S$ has a total-order relation defined over its elements.
Having a total-order relation infers the mode can be found via comparisons. 
The algorithm returns $(R, S)$, where $R$ is the mode set with respect to the elements of $S$.

\begin{algorithm}
\caption{Elicit mode via comparisons and the d-Cross Polytopes over well-defined partial orderings}
\begin{algorithmic}
\Require $M = \{ (L^{2}_{\varphi_j},\psi_{j}^{P^{\oplus}}) \}_{j=1}^{m}$
\State Learn a model $h_j:\X\to \reals^{d}$ for each instance  $(L^{2}_{\varphi_j},\psi_{j}^{P^{\oplus}}) \in M$
\State For some fixed $x\in\X$, collect all $B_j \gets \psi_{j}^{P^{\oplus}}(h_j(x))$ where $B_j\in\B_j$ 
\State Build $M \in \{0,1\}^{n\times n}$ binary relation table with provided $\{B_j \}_{j=1}^{m}$ as such
\begin{itemize}
    \item Label rows top to bottom by $y_1,\dots ,y_n$ and columns left to right by $y_1,\dots ,y_n$.
    \item For all $(\cdot, p_{y_i}, p_{y_k}) \in B_j$, if $p_{y_i} \leq p_{y_k}$ set $M[i,k]=1$ and $0$ otherwise.
\end{itemize}
\State Select largest subset $S\subseteq \Y$ such that $M_S$ is reflexive, antisymmetric, transitive, and strongly connected.
\State {Report $ (R, S) \gets \text{FindMaxElements-of-$S$}( M; S)$}
\end{algorithmic}
\end{algorithm}


\end{document}